\pdfoutput=1

\documentclass[11pt]{article}

\usepackage[preprint]{acl}

\usepackage{times}
\usepackage{latexsym}
\usepackage{relsize}
\usepackage[T1]{fontenc}

\usepackage[utf8]{inputenc}

\usepackage{microtype}

\usepackage{inconsolata}

\usepackage{graphicx}

\usepackage{xcolor}
\usepackage{amsxtra}
\usepackage{multirow}
\usepackage{diagbox}
\usepackage{longtable}
\usepackage{array}
\usepackage{stfloats}
\usepackage{balance}
\usepackage{multicol}
\usepackage{amsmath, bm}
\usepackage{booktabs} 
\usepackage{arydshln}
\usepackage{xspace}
\usepackage{enumitem}
\usepackage{threeparttable}
\usepackage{dsfont}
\usepackage{hyperref}
\usepackage{bbm}
\usepackage{makecell}
\usepackage[linesnumbered,ruled,vlined]{algorithm2e}
\usepackage{algorithmic}
\usepackage{color} 
\usepackage{subfigure}
\usepackage{stfloats}
\usepackage{subcaption}
\usepackage{mdframed}

\usepackage{listings}
\usepackage{appendix}
\usepackage{titletoc}
\usepackage{amsthm}

\setlist[itemize]{noitemsep, topsep=0pt} 

\newcommand{\tabincell}[2]{\begin{tabular}{@{}#1@{}}#2\end{tabular}}

\newcommand{\hide}[1]{} 
\newcommand{\vpara}[1]{\vspace{0.05in}\noindent \textbf{#1 }}

\newcommand{\beq}[1]{\vspace{-0.03in}\begin{equation}#1\end{equation}\vspace{-0.03in}}
\newcommand{\beqn}[1]{\vspace{-0.04in}\begin{eqnarray}#1\end{eqnarray}\vspace{-0.04in}}
\newcommand{\beal}[1]{\vspace{-0.03in}\begin{align}#1\end{align}\vspace{-0.03in}}

\newcommand{\model}{\textsc{TableLLM}\xspace}
\newcommand{\smodel}{\textsc{TableLLM} }

\theoremstyle{plain}
\newtheorem{lemma}{Lemma}[section]  

\theoremstyle{plain}
\newtheorem{theorem}{Theorem}[section]  

\title{\model: Enabling Tabular Data Manipulation by LLMs in Real Office Usage Scenarios}

\author{
Xiaokang Zhang$^{1,3}$\thanks{Equal Contributions.}\thanks{Work was done when interned at Zhipu AI.},
Sijia Luo$^{1}$\footnotemark[1],
Bohan Zhang$^{1,3}$,
Zeyao Ma$^{1,3}$,
Jing Zhang$^{1,3}$\thanks{Corresponding Author.},\\
Yang Li$^{1,3}$,
Guanlin Li$^{1,3}$,
Zijun Yao$^{2}$,
Kangli Xu$^{2}$,
Jinchang Zhou$^{2}$,\\
Daniel Zhang-Li$^{2}$,
Jifan Yu$^2$,
Shu Zhao$^4$,
Juanzi Li$^2$,
Jie Tang$^2$
\\[3pt]
$^1$School of Information, Renmin University of China,
$^2$Tsinghua University,
\\
$^3$Key Laboratory of Data Engineering and Knowledge Engineering, Beijing, China
\\
$^4$Computer Science, Anhui University, China
\\[1pt]
\small \texttt{\{zhang2718, luosijia0906, zbhmint, zeyaoma, zhang-jing\}@ruc.edu.cn}
}

\begin{document}
\maketitle

\newtheorem{problem}{Problem}
\begin{abstract}
We introduce \model, a robust large language model (LLM) with 8 billion parameters, purpose-built for proficiently handling tabular data manipulation tasks, whether they are embedded within documents or spreadsheets, catering to real-world office scenarios.
We propose a distant supervision method for training, which comprises a reasoning process extension strategy, aiding in training LLMs to understand reasoning patterns more effectively as well as a cross-way validation strategy, ensuring the quality of the automatically generated data.
To evaluate the performance of \model, we have crafted benchmarks tailored to address both document and spreadsheet formats as well as constructed a well-organized evaluation pipeline capable of handling both scenarios. Thorough evaluations underscore the advantages of \smodel when compared to various existing general-purpose and tabular data-focused LLMs.
We have publicly released the model checkpoint, source code, benchmarks, and a web application for user interaction\footnote{\url{https://tablellm.github.io/}}.

\end{abstract}

\section{Introduction}
\label{sec:introduction}
\begin{figure}[t]
  \centering
  \includegraphics[width=0.48\textwidth]{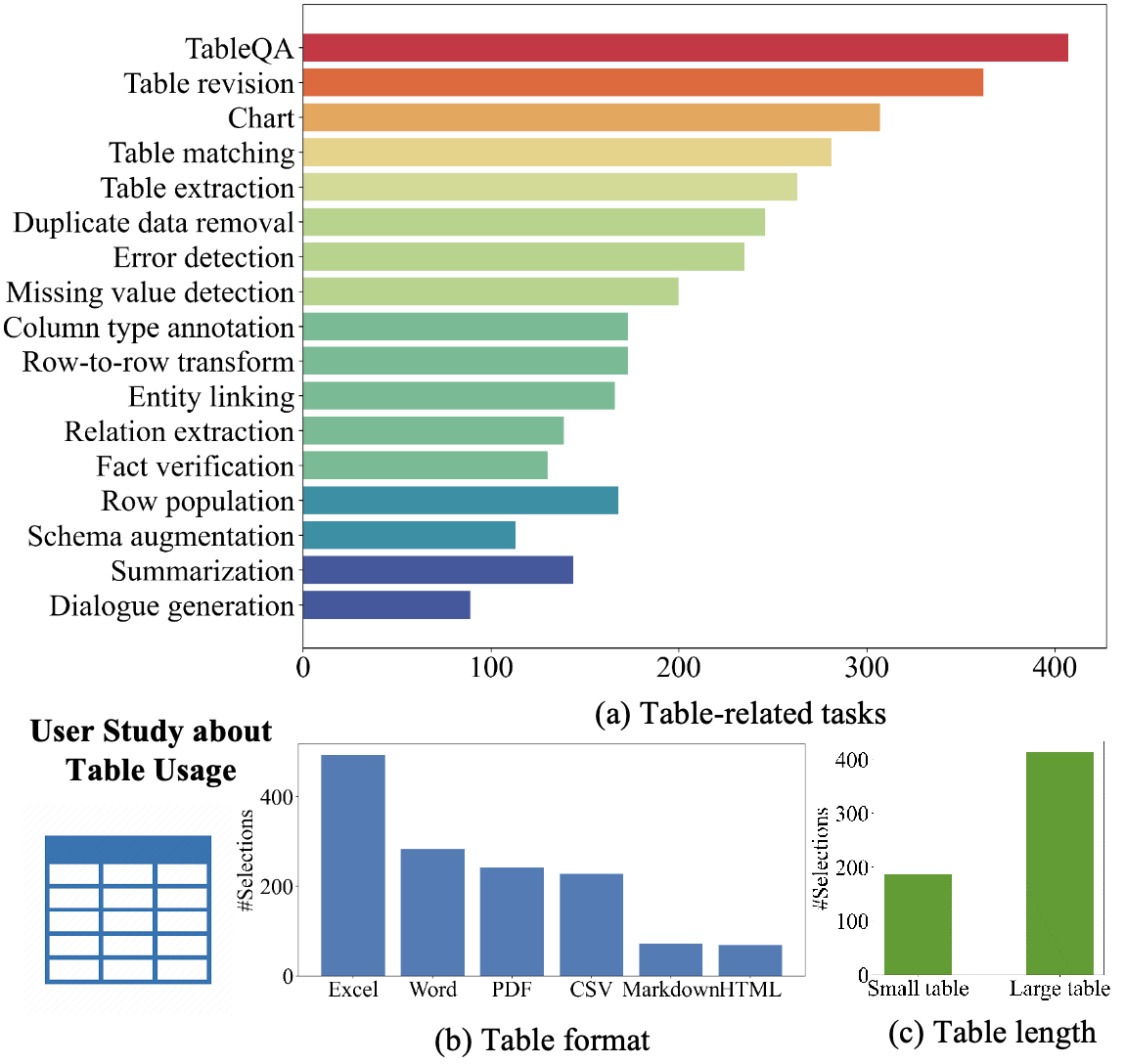} \caption{Illustration of the user study about (a) table-related tasks (tableQA, table revision, chart, table matching, duplicate data removal, etc.); (b) table formats (Excel, Word, etc.); (c) table length (Small: $<50$ rows, Large: $\geq 50$ rows).} 
 \label{fig:survey} 
\end{figure}

Large amounts of data are organized in tabular form and widely used in various scenarios. However, working with tabular data can be challenging, as many table-related tasks are laborious, error-prone, and require specialized skills. Automating these tasks offers significant benefits to both academic and industrial sectors, attracting considerable interest~\cite{badaro2023Transformers,dong2022Table}.

To capture insights from office users, we conduct an extensive user study involving a questionnaire distributed to 507 diverse paiticipants, focusing on table-related tasks. Details about the survey are presented in Appendix~\ref{sec:surveyexplanation}. As shown in Figure~\ref{fig:survey}, participants preferred tasks involving tableQA, revision, chart creation, and matching, primarily using Excel/CSV and Word/PDF formats, including long tables. These findings highlight two characteristics of real-world office use compared to academically-focused table tasks. \textbf{(1) Diverse Operations:} user preferred tasks involve a wide range of operations, including query, update, merge, and chart, which go beyond tableQA task. \textbf{(2) Unique Processing Approaches for Different Formats:} Word/PDF documents often contain contextual textual data, allowing for hybrid querying. Excel/CSV spreadsheets contain regular and long tables, enabling more intricate operations like update and merge.

Previous studies have focused on improving a model's reasoning capabilities for table question answering. Moving beyond tableQA, some of these endeavors have also tackled diverse table-related tasks, including fact verification~\cite{DATER,TableLlama,reactable,structgpt,tapex}, column type annotation~\cite{li2023tablegpt,TableLlama}, table matching~\cite{li2023tablegpt}, schema augmentation~\cite{li2023tablegpt,TableLlama}, and more. 
However, existing methods for handling tabular data in real-world office scenarios have limitations. Some use LLMs to directly extract answers from internal parameters, which is effective for document-embedded tables but inadequate for long tables and diverse spreadsheet operations. Others focus on writing and executing code for spreadsheets but struggle with hybrid queries combining text and tabular data.

Based on this insight, We present \model, specifically designed to handle a wide array of table operations in spreadsheet and document usage scenarios, named tabular data manipulation in real office usage scenarios. 
To facilitate model training, we introduce a distant supervision method that complements the reasoning process of existing benchmarks, aiding in training LLMs to understand reasoning patterns more effectively. Additionally, we validate the automatically generated training data through a cross-way validation strategy, ensuring data quality. We also provide a theoretical analysis of the effectiveness of cross-way validation compared to self-check and same-way validation.
Utilizing this training data, we fine-tune Llama3.1(8B)~\cite{llama3-report}, resulting in the development of \model. This model adeptly handles document-embedded tabular data through an inner-parameter-driven approach and spreadsheet-embedded tabular data via a code-driven method.

A rigorous performance assessment is conducted, involving the collection of primary tableQA test instances from existing benchmarks and the creation of additional table manipulation instances by an annotation team. Given the complex evaluation process under the two scenarios, we design a meticulous evaluation method that considers query, update, merge and chart operations.
\textbf{\smodel proves to be on par with the most capable commercial LLM GPT-4o in the document-embedded scenario, and even outperforms GPT-4o in the spreadsheet-embedded scenario.}

In the realm of tabular data processing research, our contributions encompass:
(1) Addressing a practical problem of tabular data manipulation in real-world office usage scenarios.
(2) Presenting techniques that extend reasoning processing and integrate a cross-way validation strategy to enhance the quality of distant supervision training data. 
(3) Delivering a high-quality open-source LLM tailored for tabular data manipulation in 8B, thereby enhancing accessibility and fostering collaboration within the community.
(4) Offering an online application service to facilitate convenient usage.

\section{Related Work}
\label{sec:related}

\vpara{Table Question Answering.}
Beyond the primary tableQA task, various research endeavors tackle basic table analysis tasks such as fact verification, column type annotation and schema augmentation, typically involving web-extracted tables of relatively short length with textual content. This research has evolved through three main approaches: (1) \textbf{Representation Learning}: Many traditional methods, such as TaBERT~\cite{tabert}, TAPAS~\cite{tapas}, TableGPT~\cite{tablegpt2020}, Tabbie~\cite{iida2021tabbie} and TAG-QA~\cite{zhao2023localize}, focus on intricate encoder designs with various positional encodings and attention mechanisms, while TAPEX~\cite{tapex} and GraPPa~\cite{yu2020grappa} integrate SQL execution as a pre-training task; (2) \textbf{Finetuning LLM}: Researchers leverage LLM to train unified models like TableLlama~\cite{TableLlama}, TAT-LLM~\cite{zhu2024tat} and TableGPT2~\cite{su2024tablegpt2} on multiple table-related benchmarks, with UnifiedSKG~\cite{unifiedskg} extending this to structured data tasks; and (3) \textbf{Prompting LLM}: Researchers develop multi-step prompting strategies for GPT series models, employing tools like SQL and Python, such as DATER's SQL usage~\cite{DATER}, StructGPT's self-defined interfaces~\cite{structgpt}, and Binder's result integration method~\cite{Binder}.

\vpara{Table Manipulation.}
A new research direction aims to enhance table manipulation capabilities, particularly focusing on tasks such as insert, update, and delete operations within spreadsheet formats like Excel and CSV, as well as databases~\cite{sheetcopilot,c3,dinsql,datacopilot}. Such tasks often involve working with lengthy and regular tables, making it practical to utilize LLMs alongside tools to address them. For instance, DB-GPT~\cite{xue2023db}, ChatDB~\cite{hu2023chatdb}, C3~\cite{c3} and Din-SQL~\cite{dinsql} translate questions into SQL queries. SheetCopilot~\cite{sheetcopilot} and DataCopilot~\cite{datacopilot} develop their atomic interfaces based on Excel's embedded functions and various programming languages, allowing LLMs to invoke them. 

\model aims to tackle both table question answering and manipulation, making it more applicable to real-world user scenarios. Methodologically, \model fine-tunes a foundation model using synthetic data. We highlight the proposed cross-validation method, specifically designed to accommodate the unique characteristics of tables, ensuring the reliability of the synthetic data.

\section{Problem Definition}
\label{sec:problem}
\begin{figure}[t]
  \centering
\includegraphics[width=0.36\textwidth]{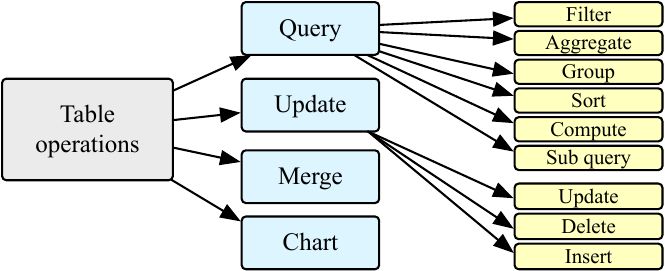} \caption{Common operations for table manipulation.} 
 \label{fig:operation} 
\end{figure}

\vpara{Tabular Data} refers to data organized in table or grid format. 
On top of it, \textbf{document-embedded tabular data} is tabular data integrated into documents, often in Word/PDF files, while \textbf{spreadsheet-embedded tabular data} refers to tables within spreadsheets, typically in Excel/CSV files.

\vpara{Operation Definition.}
In light of the user study, tabular data manipulation tasks can be categorized into four primary operations: query, update, merge, and chart, as detailed in Figure~\ref{fig:operation}. The ``query'' operation selects desired data, encompassing filter, aggregate, group, and sort functions, covering most tableQA scenarios. The ``update'' operation modifies, deletes or adds data, while the  ``merge'' operation combines two tables. Lastly, the ``chart'' operation visualizes data using bar, pie or line charts. 

\begin{problem}
\textbf{Tabular Data Manipulation in Real Office Usage Scenarios} focuses on developing an LLM that can perform a range of query, update, merge, and chart operations with tabular data embedded in documents and spreadsheets. 
\end{problem}

For document-embedded tabular data, querying specific information is the primary need, whereas spreadsheet-embedded tabular data often demand querying, data modification, and chart generation. 

\section{\model}
\label{sec:aproach}

\begin{figure*}[t]
  \centering
\includegraphics[width=\textwidth]{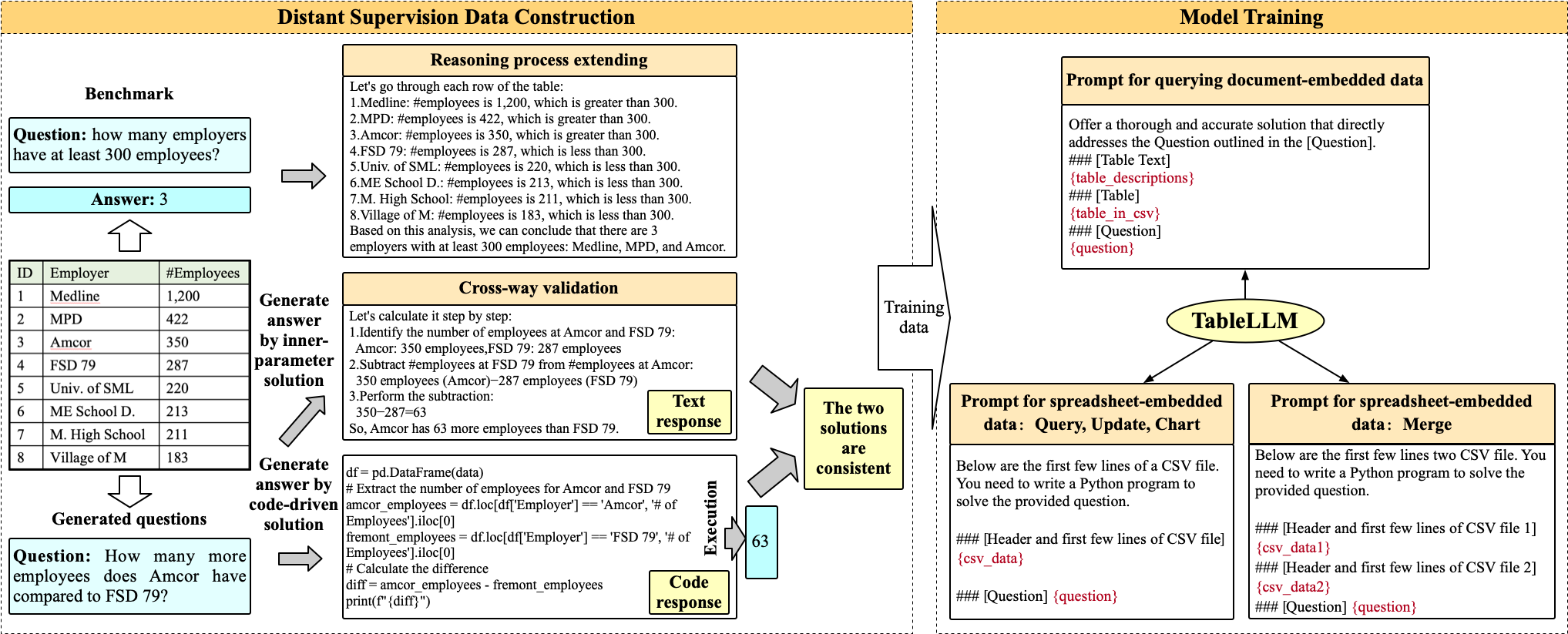} \caption{Overview illustration of \model. The construction of distant supervision data involves two key steps: (1) expanding the reasoning processes based on (question, answer) pairs from existing benchmarks, and (2) cross-way validation of generated (question, answer) pairs. Model training necessitates unique prompts tailored to operations in different scenarios.} 
 \label{fig:overview} 
\end{figure*}

The overview design of \model is shown in Figure~\ref{fig:overview}, which consists two primary aspects:
\textbf{(1) Distant Supervision Data Construction.}
The development of distant supervision data involves the integration of both existing benchmark training data and new questions and answers generated from available tabular data. To enhance the training of LLMs, we suggest expanding the reasoning processes within benchmark data. 
Additionally, to assure the quality of the automatically generated training data, we introduce a cross-way validation strategy which utilizes diverse solution methods for cross-validation. 
\textbf{(2) Model Training.}
The training process utilizes distinct prompts for document-embedded and spreadsheet-embedded tabular data.

\subsection{Distant Supervision Data Construction}
\label{sec:data_construction}

\vpara{Extending Reasoning Process for Existing Benchmarks.}
While existing benchmarks offer ample training data for tableQA, the simple short answers provided by individual instances fall short for tackling complex tabular data manipulation tasks, which often demand intricate reasoning processes to derive answers effectively. Therefore, we augment existing benchmarks by enriching their reasoning processes to facilitate the model training.

Primarily, to address queries on document-embedded tabular data, we gather training data from widely-adopted tableQA benchmarks including WikiTQ~\cite{WTQ}, FeTaQA~\cite{Nan2021FeTaQAFT}, and TAT-QA~\cite{TAT-QA}. Inspired by CoT~\cite{cot}, We extend the provided short answers by presenting GLM-4-Plus~\cite{glm2024chatglm} with the (question, answer) pairs and instructing it to enhance the reasoning process. This augmentation is represented in textual form, rather than as code, to align with the nature of queries involving hybrid text and tabular data inputs. 
Notably, for WikiTQ and FeTaQA solely provide tabular data, we supplement them by generating table descriptions using GLM-4-Plus. Due to the inner-parameter-driven technique employed, we impose a constraint on the length of input tables, limiting them to a token count of fewer than 500. To validate the quality of these text-based reasoning processes, we refer to the evaluation prompt of CritiqueLLM~\cite{critiquellm}, an LLM specialized in rating, and use DeepSeek-V3~\cite{liu2024deepseek} to assess the consistency between the reasoning process and the answers provided in the benchmarks. 

Furthermore, to handle queries on spreadsheet-embedded tabular data, we compile training data from two Text2SQL benchmarks: WikiSQL~\cite{zhongSeq2SQL2017} and Spider~\cite{Spider}. Given that spreadsheet-embedded tabular data manipulation primarily involves pure tabular data inputs and complex table manipulations, it aligns more with code-driven techniques. Thus, we select training instances from WikiSQL and Spider, as they correspond to SQL queries. However, instead of directly using the provided SQL queries, we expand pandas code as the reasoning process for each (question, answer) pair by Deepseek~\cite{bi2024deepseek}, a recent powerful code LLM, as Pandas offers greater flexibility to support functionalities such as chart beyond query, update, and merge. We ensure the quality of the generated code by validating that the executed outcomes align with the provided answers in the benchmarks.
Note for Spider, in line with our focus on single-table operations typical in office scenarios, we exclude multi-table queries and those whose SQL queries yield null results to better reflect real-world applications.

\vpara{Automatically Generating Training Data by Cross-way Validation.}
While the training data derived from existing benchmarks is of high quality, the variety of questions and answers, especially the table update, merge, and chart operations they offer is limited. To address this, we introduce a cross-way validation strategy for automatically generating new questions and answers using only the provided tabular data. The process is as follows:

\textbf{(1) Question Generation.} We select 5,177 tables from WikiTQ, 5,000 from TAT-QA, and 4,019 from FeTaQA with less than 500 tokens to simulate document-embedded tabular data. For each table, GLM-4-Public generates questions involving single or multiple table query operations, as depicted in Figure~\ref{fig:operation}. GLM-4-Public also creates contextual table descriptions for WikiTQ and FeTaQA-sourced tables, while TAT-QA tables retain their original text context. Furthermore, we select 1,300 long tables from GitTables~\cite{hulsebos2023gittables}. For each table, we generate 10 questions involving various table manipulation operations, as illustrated in Figure~\ref{fig:operation}. Existing benchmarks typically focus on table query operations, so update, chart, and merge operations are all generated. For query, update, and chart operations, we prompt GLM-4-Public for question generation. However, for the merge operation, given its well-defined nature, we directly construct templates to generate the merge question.
Appendix~\ref{sec:prompt-self-create} provides the prompts and templates used for question generation.

\textbf{(2) Answer Generation and Cross-way Validation.} 
We utilize LLMs to generate answers for the questions derived from document-embedded tables and propose a cross-validation method to ensure quality. This approach leverages the unique characteristics of tables, where answers can be obtained in two distinct ways: directly from the LLM internal parameters or by generating and executing Pandas code. By comparing the results from both methods, we ensure the reliability of the answers.

Specifically, for each question based on tabular data, we use GLM-4-Plus to generate answers through both internal-parameter inference and code-driven execution. We generate 10 answers via the internal-parameter approach and 50 answers via the code-driven approach. To establish a reference answer, we aggregate the 50 code-driven answers through majority voting. Since code execution results often include additional descriptions, making exact string matching impractical, we use the ROUGE-L\cite{lin2004rouge} metric to calculate the text similarity among the code results, cluster them, and select the centroid of the largest cluster as the reference answer.
Finally, from the 10 internal-parameter-generated answers, we select the one most aligned with the reference answer as the final output. This cross-way validation, combining internal-parameter-driven and code-driven techniques, leverages their complementary strengths to enhance answer quality and reliability.

\textbf{Theoretical Proof.} This cross-validation approach is inspired by ensemble learning~\cite{dietterich2002ensemble}, which combines multiple weak learners to create a strong learner. Building on this concept, we conduct an improved theoretical inference to ensure the quality of automatically generated data.
Let's denote $Y_a$ as the event that the first response is correct, $Y_b$ as the event that the second response is correct, $Y$ as the event that both responses are correct, and $E$ as the event that the two responses are consistent. Based on these definitions, we can establish the following theorem:

\begin{theorem}
(1) If $A$ and $B$ are drawn from the same distribution such that $P(Y_a) = P(Y_b) = p > 1/2$, then consistency checking outperforms single inference, i.e., $P(Y|E) \geq P(Y_a)$. (2) If \( A \) and \( B \) are further drawn from independent distributions, the effect will be superior (in terms of expectation).
\end{theorem}

This theorem suggests that when the model's probability of providing correct answers exceeds 1/2, employing consistency verification is decisively more effective than direct inference. Moreover, in terms of expected performance, utilizing cross-validation with two independent distributions surpasses consistency checks with a single distribution. The proof is provided in Appendix~\ref{sec:cross-way-verification}.

For questions on spreadsheet-embedded tabular data, we employ GLM-4-Plus to generate a code solution, which is followed by the generation of an alternative code solution using GLM-4-Plus again. The accuracy of the executed outcomes from the first code are verified by comparing them with the outcomes of the second code. Given the potential diversity of two coding solutions yielding the same answers, this dual-coding strategy can be regarded as stemming from different distributions. Thus it also functions as a cross-way validation method.

\subsection{Model Training}
In the scenario of document-embedded tabular data, the input for LLMs includes both the text and the entire content of the table. However, in the case of spreadsheet-embedded tabular data, due to the typically extensive length of the table, only the header and a subset of rows are provided as input to the LLM. The prompt for the merge operation is specifically designed, illustrated in Figure~\ref{fig:overview}.

Given the prompt $x$ as input, we enable LLMs to generate either the textural or code solution, collectively denoted as $y$. We hybridize the document-embedded and spreadsheet-embedded training data in a 1:1 ratio, thoroughly shuffle them, and then partition them into batches for training.



The trained single model addresses two types of data sources. Given that code models tend to excel in reasoning-intensive tasks compared to text models~\cite{liang2023holistic}, combining the two data sources can enhance text-level reasoning with code-level reasoning. Moreover, a single model could alleviate deployment pressure.

\subsection{Model Deployment as Web Application}


We launch our \smodel as a web application, with a screenshot shown in Figure~\ref{fig:app}. Users can upload tabular data from documents (Word, PDF) and spreadsheets (Excel, CSV). The system parses PDF and Word files into CSV for visualization. Users enter queries, and the \smodel generates responses—tables, charts, or text—based on prompts and document type. It also supports table merging, allowing users to merge two spreadsheets with specified conditions\footnote{Currently, the system is configured to support the merging of two tables only.}. Details are in Appendix~\ref{sec:webapp}.

\section{Experiment}
\label{sec:experiment}
\begin{table}[t]
  \newcolumntype{?}{!{\vrule width 1pt}}
  \newcolumntype{C}{>{\centering\arraybackslash}p{2em}}
  \centering
  \caption{Benchmark (test set) statistics}
  \label{tb:benchmark}
  \small
  \scalebox{0.76}{
  \begin{tabular}{c?c?c?c}
    \toprule
    Scenario &	Name &	Description & Size \\
     \midrule
 
     \multirow{3}{*}{\tabincell{c}{Document\\-embedded}}  
     & WikiTQ  & <500 tokens \& add text & 633\\
     & FeTaQA & <500 tokens \& add text & 753\\
     & TAT-QA & <500 tokens &800\\

     \midrule
     \multirow{3}{*}{\tabincell{c}{Spreadsheet\\-embedded}}  
     & WikiSQL & Remove vague questions &1,000 \\
     & Spider  & Choose single table & 512  \\
     & Our created &Query/Update/Merge/Chart &1,200 \\
    \midrule
    Both & \model-bench & - & 4,898 \\
    \bottomrule

  \end{tabular}
  }
\end{table}

\begin{table*}[t]
\newcolumntype{?}{!{\vrule width 1pt}} 
	\newcolumntype{C}{>{\centering\arraybackslash}p{2em}}
	\caption{
		\label{tb:overall} Overall evaluation in both document-embedded and spreadsheet-embedded tabular data scenarios (\%)
	}
	\centering 
	\small
	\renewcommand\arraystretch{1.0}
\scalebox{0.82}{
\begin{tabular}{l?ccc?ccc?c?c}
\toprule
\multirow{2}{*}{Model}  
& \multicolumn{3}{c?}{Document-embedded tabular data}   & \multicolumn{3}{c?}{Spreadsheet-embedded tabular data} & \multirow{2}{*}{Average accuracy} & \multirow{2}{*}{Inference times} \\
\cmidrule{2-4} \cmidrule{5-7}     
& WikiTQ & TAT-QA & FeTaQA & WikiSQL & Spider & Our created & &\\
\midrule
TaPEX &38.55 &--&--&83.90 & 15.04 & --& 45.83& 1\\
TaPas & 31.60 & --& --& 74.20& 23.05 & --&42.95&1\\
TableLlama &24.01 &22.25 &20.47 & 43.70&--&--& 23.36&1\\
TableGPT2(7B) &77.25 &88.12 &75.58 &63.0 &77.34 &74.42&75.95 &1\\
\midrule
Llama3.1(8B) &71.88 &74.25 &83.40 &40.60 &18.75 &43.17 &55.34 &1\\
GPT-3.5 &58.45 &72.13 &71.18 &81.70 &67.38 &77.08 & 69.82& 1\\
GPT-4o &\textbf{91.47} &\textbf{91.50} &\textbf{94.42} &\underline{84.00} &69.53 &\underline{77.83} &\underline{84.79} &1\\
\midrule
CodeLlama (13B) &43.44 &47.25 &57.24 &38.30 &21.88 &47.58 &43.63 &1\\
Deepseek-Coder (33B) & 6.48&11.00 &7.12 &72.50 &58.40 &73.92 &33.84&1\\
\midrule
StructGPT (GPT-3.5) &52.45 &27.53 &11.80 &67.80 &\textbf{84.80} &--&43.06&3\\
Binder (GPT-3.5) &61.61 &12.77 &6.85 &78.60 &52.55 &--&36.25&50\\
DATER (GPT-3.5) &53.40 &28.45 &18.26 &58.20 &26.52 &--& 32.98&100\\
\midrule
\model (8B)& \underline{89.10} & \underline{89.50} & \underline{93.36} & \textbf{89.6}
 & \underline{81.05} &\underline{77.83} &\textbf{86.74} &1\\
\bottomrule
\end{tabular}
}
\begin{tablenotes}
  \item * Underline represents the runner up.
\end{tablenotes}
\end{table*}


\subsection{Test Set Creation}
We collect test sets from established benchmarks for document-embedded query tasks, including WikiTQ~\cite{WTQ}, FeTaQA~\cite{Nan2021FeTaQAFT}, and TAT-QA~\cite{TAT-QA}. For spreadsheet-embedded table tasks, we utilize test sets from WikiSQL~\cite{zhongSeq2SQL2017} and Spider~\cite{Spider}, which align with our query operation requirements. We extract the table, question, and answer from each instance. 

As no benchmarks exist for table update, merge, and chart operations, we create test set through human annotation. We choose 50 long tables from InfiAgent-DABench~\cite{infiagent}, ensuring they are entirely distinct from our training data. Following the process outlined in Section~\ref{sec:data_construction}, we generate questions and answers. An annotator team verifies the generated content, including answers, codes, and operation types. 
Since initial questions lack linguistic diversity, we utilize five prominent models from Huggingface for rewriting to enhance variety. This results in a composition of 10\% original questions and 90\% rewritten ones, with each model contributing to 18\% of the rewrites. 
 Table~\ref{tb:benchmark} displays the benchmark statistics. 

\subsection{Evaluation Approach}
Given the diverse range of operation types in our dataset, we have adopted a categorized evaluation approach across different operations:

\begin{itemize}[leftmargin=*]
    \item \textbf{Query operations:} For the answers obtained through code execution, we conduct an exact match comparison between the model's output and the ground truth answers. For answers inferred directly via inner-parameters, we use DeepSeek-V3 to assign a score from 1 to 10 by comparing the model's output with the ground truth answers, with a score threshold of 7 considered correct, as discussed in Section~\ref{sec:data_construction}. This is because the generated answers are often lengthy and challenging to precisely match.
    We also conduct a meta evaluation on DeepSeek-V3's rating scores by humans, as detailed in Appendix~\ref{sec:critiqueds}.

    \item \textbf{Update and merge operations:} As these operations directly modify tables, we require the model's output to be the complete modified table. We then perform an exact match comparison between the model's output and the ground truth answers to determine correctness.
    
    \item \textbf{Chart operations:} Assessing charting operations is challenging through direct answer comparison. Instead, we compare model code output with ground truth code. DeepSeek-V3 is once again employed to compare the model's output code with the ground truth code, using a score threshold of 5 for evaluation.

\end{itemize}

Based on the correctness determination, we assess accuracy.

\subsection{Comparison Methods}
Comparison methods are divided into four types: 
\begin{itemize}[leftmargin=*]
    \item \textbf{Pre-training and fine-tuning LLMs:} This category encompasses models like TaPas~\cite{tapas} (based on BERT) TAPEX~\cite{tapex} (based on BART),  TableLlama~\cite{TableLlama} (based on Llama2 (7B)) and TableGPT2(7B)~\cite{su2024tablegpt2}.
    
    \item \textbf{General LLMs:} This group includes GPT-3.5~\cite{InstructGPT}, GPT-4o~\cite{gpt4}, and Llama3.1 (8B)~\cite{llama3-report}.
    
    \item \textbf{Coding-specific LLMs:} This category contains LLMs tailored for coding tasks, including CodeLlama~\cite{CodeLlama} and DeepSeek~\cite{bi2024deepseek}.

    \item \textbf{Prompt-driven LLMs:} This group includes StructGPT~\cite{structgpt}, ReAcTable~\cite{reactable}, Binder~\cite{Binder}, and DATER~\cite{DATER}, focusing on creating sophisticated prompts to guide LLMs in processing tabular data.
\end{itemize}


\vpara{Implementation.}
(1) TaPas and TAPEX have individual checkpoints trained on WikiTQ and WikiSQL. We assess their performance in document-embedded tabular data scenarios using the WikiTQ-trained versions and in spreadsheet-embedded tabular data scenarios using the WikiSQL-trained versions. As for TableLlama and TableGPT2, we evaluate its unique checkpoints directly. (2) For both general and coding-specific LLMs, we provide customized prompts for scenarios to handle document-embedded and spreadsheet-embedded tabular data, as detailed in Appendix~\ref{sec:prompt_baselines}. (3) 
Prompt-driven LLMs follow their established prompts.
We unify StructGPT's prompts for WikiTQ, TAT-QA, FeTaQA, and WikiSQL to match WikiTQ's format.
Meanwhile, Binder and DATER use a single unified set of prompts across all benchmarks.
(4) \smodel is trained using Llama3.1 (8B). During inference, we consistently apply the same set of prompts used during training phase. The generated training data is presented in Table~\ref{tb:training} in Appendix~\ref{sec:trainingdata}.

\subsection{Overall Experimental Results}
\vpara{Effectiveness.}
Table~\ref{tb:overall} displays the overall evaluation in two scenarios. ``--'' in the table indicates that the method does not support the dataset or that the tested accuracy is too low. The results show that \textbf{\smodel generally surpasses others in the spreadsheet-embedded scenario and is on par with GPT-4o in the document-embedded scenario}. Detailed findings include:

\textbf{(1) TaPEX and TaPas show limited performance due to their small model sizes.} These two pre-training and fine-tuning models only demonstrate relatively strong performance on WikiSQL and WikiTQ benchmarks.

\textbf{(2) StructGPT, Binder, and DATER's varying performance across datasets suggests a limitation in the generalization capability of prompt-driven LLMs.} While these models consistently perform well in the WikiTQ benchmark, their performance weakens on other datasets. StructGPT stands out in the Spider benchmark due to its customized prompts tailored for this specific dataset.

\textbf{(3) DeepSeek (33B) excels in the spreadsheet-embedded tabular data scenario.} This superior performance is attributed to its extensive optimization for coding capabilities. However, this specialization in coding comes at the expense of direct answer inference from inner parameters when dealing with document-embedded tabular data.

\textbf{(4) Our \smodel outperforms GPT-3.5 and GPT-4o in the spreadsheet-embedded  scenario.} Moreover, in our created benchmark with entirely distinct tabular data and questions from the training data, \smodel achieves an impressive 77.83\% accuracy, showcasing robust generalization ability. Conversely, in the document-embedded scenario, \smodel performs close to GPT-4o and significantly better than GPT-3.5, possibly due to the scenario's demand for extensive commonsense reasoning with text data, where \smodel could benefit from enhanced training in text QA. 

\vpara{Efficiency.}
All methods, except prompt-driven LLMs, require only one inference process per instance. However, Binder necessitates a one-step inference for each instance, requiring 50 samples per step for self-consistency validation. DATER requires four-step inferences for each instance, with self-consistency validation at each step, totaling 100 inferences per instance. StructGPT requires three inferences per question.

\subsection{Ablation Studies on Training Data}
\label{sec:ablationTrain}

\begin{table}[t]
  \newcolumntype{?}{!{\vrule width 1pt}} 
  \newcolumntype{C}{>{\centering\arraybackslash}p{2em}} 
  \centering
  \caption{Effect of diverse training data sources (\%)}
  \label{tb:trainingsources}
  \small
  \scalebox{0.8}{
  \begin{tabular}{l?cc?cc}
    \toprule
    \multirow{2}{*}{Train data}&\multicolumn{2}{c?}{Document-embedded}&\multicolumn{2}{c}{Spreadsheet-embedded}\\
    \cmidrule{2-3}
    \cmidrule{4-5}
&  WikiTQ&TAT-QA& Spider& Our created\\
    \midrule
     Llama3.1 (8B)          &71.9 & 74.3  & 18.8 &43.2\\
     Original train data      &71.3 &68.5 & -- &--\\
     Extended train data      &83.1 &87.5 &77.9 &55.8 \\
     Generated train data  &82.2 &86.3 &62.1 &72.2 \\
     Mixed data  &84.0 &88.3 &80.9 &73.6 \\
     \bottomrule
  \end{tabular}
  }
\end{table}
\begin{table}[t]
  \newcolumntype{?}{!{\vrule width 1pt}} 
  \newcolumntype{C}{>{\centering\arraybackslash}p{2em}} 
  \centering
  \caption{Effect of cross-way validation strategy on document-embedded tabular data (\%)}
  \label{tb:cross-way-performance}
  \small
  \scalebox{0.97}{
  \begin{tabular}{l?cc}
    \toprule
    Validation strategy &WikiTQ  &TAT-QA\\
     \midrule
     Self-check validation  &74.5  &81.3   \\
     Same-way  validation      &77.8 &85.8  \\
     Cross-way validation     &84.7& 87.9  \\
     \bottomrule
  \end{tabular}
  }
\end{table}
\vpara{Effect of Diverse Training Data Sources.}
We analyze the influence of different training datasets by comparing five distinct training configurations:

\begin{itemize}[leftmargin=*]
\item \textbf{Llama3.1 (8B)}: The base version without any training.
\item \textbf{With original training data of existing benchmarks}: Train Llama3.1 using 2,000 training instances from TAT-QA and WikiTQ, then evaluate on corresponding test sets. 
\item \textbf{With extended training data of existing benchmarks}: Train on 2,000 training instances from WikiTQ and TAT-QA, supplemented with extended reasoning process. Train on 2,000 training instances from Spider, supplemented with extended code, and evaluate on both Spider and our created test sets.
\item \textbf{With generated training data}: Train on 2,000 generated instances based on WikiTQ/TAT-QA's tabular data, then test on corresponding test sets. Train on 2,000 generated code-outputted instances based on GitLab's tabular data, and evaluate on Spider and our created test sets.
\item \textbf{With mixed data}: Train with a mix of 2,000 extended and 2,000 generated instances from TAT-QA/WikiTQ, then evaluate on corresponding test sets. Train with a mix of 2,000 extended Spider training instances and 2,000 generated code-outputted instances, and evaluate on both Spider and our created test sets.
\end{itemize}

As shown in Table~\ref{tb:trainingsources}, including the original training data leads to worse performance than the base version without training. This may be because the earlier benchmark data is too simplistic, hindering LLM training. The results demonstrate the effectiveness of incorporating extended reasoning processes, showcasing a performance boost of 15.6\% and 17.8\% on WikiTQ and TAT-QA respectively, compared to the base version. This improvement is mainly due to the inclusion of detailed explanations, helping LLMs recognize reasoning patterns. Furthermore, the addition of generated data yields an additional 14.3\% and 16.2\% enhancement in performance over the base version on WikiTQ and TAT-QA, respectively, emphasizing the value of automatically generated data. Notably, the combination of extended and generated training data leads to a significant 16.8\% and 18.8\% increase in performance relative to the base version, highlighting the advantages of integrating diverse data sources. The results observed on Spider and our created test sets further corroborate the benefits of extended and generated training data.

\vpara{Effect of Cross-way Validation.}
We examine the effectiveness of cross-way validation methods on document-embedded tabular data. For fair comparison, our cross-way validation method derives one answer from both the inner-parameter solution and the code-driven solution.
We compare it against two other validation methods: \textbf{Same-way validation}, which generates two direct answers by the same inner-parameter technique of GLM-4-Plus and assesses their alignment using DeepSeek-V3, and \textbf{Self-check validation}, which enables GLM-4-Plus to generate one textual solution and self-check its answer. Table~\ref{tb:cross-way-performance} presents that our cross-way validation method outperforms the other methods, due to its use of two distinct responses.

\subsection{Training Strategy Investigation}
\label{sec:trainingstragey}
We investigate two training aspects: data size and the ratio between document- and spreadsheet-embedded data. We explore the following variants:

\begin{itemize}[leftmargin=*]
 \item \textbf{Data size:} Options include 1k, 2k, 5k, 10k, 20k, and 40k instances.
 \item \textbf{Data ratio:} The proportion of document- to spreadsheet-embedded data, explored in ratios of 0:10, 2:8, 4:6, 5:5, 6:4, 8:2, and 10:0.
\end{itemize}

The default configuration for our experiments is 10k training data instances, a 5:5 data ratio. 

Figure~\ref{fig:trainingstrategy} shows \smodel's accuracies under various training data settings. As depicted in Figure~\ref{subfig:datasize}, performance gains follow a log-linear relationship with the training data size, motivating us to stop early at 40K, which offers a cost-effective balance. Figure~\ref{subfig:dataratio} indicates that a 5:5 ratio yields balanced performance across both data types. 

\begin{figure}[t]
	\centering
	\subfigure[Data size]{\label{subfig:datasize}
		\includegraphics[width=0.2\textwidth]{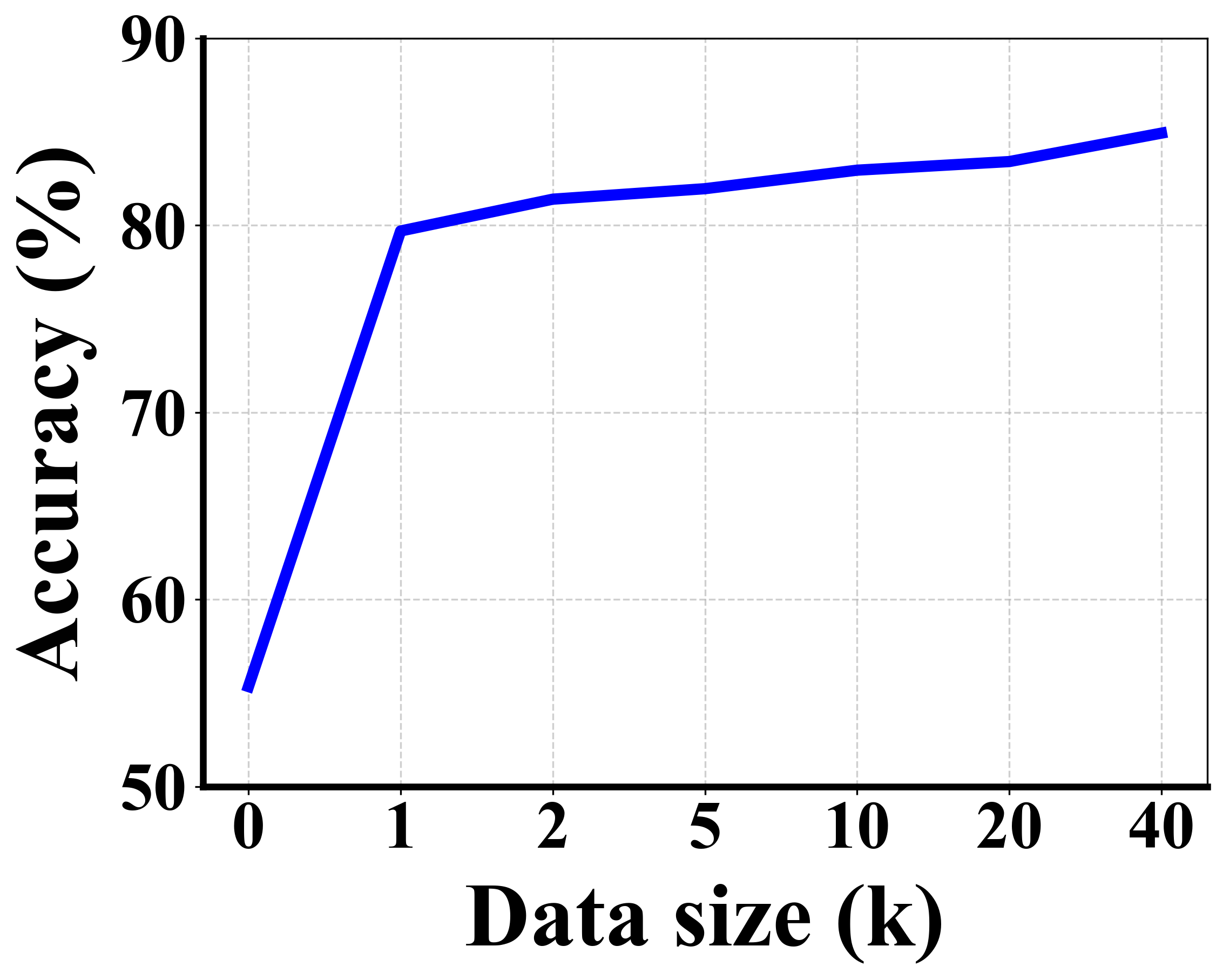}
	}
        \hspace{0.2cm}
	\subfigure[Data ratio]{\label{subfig:dataratio}
		\includegraphics[width=0.205\textwidth]{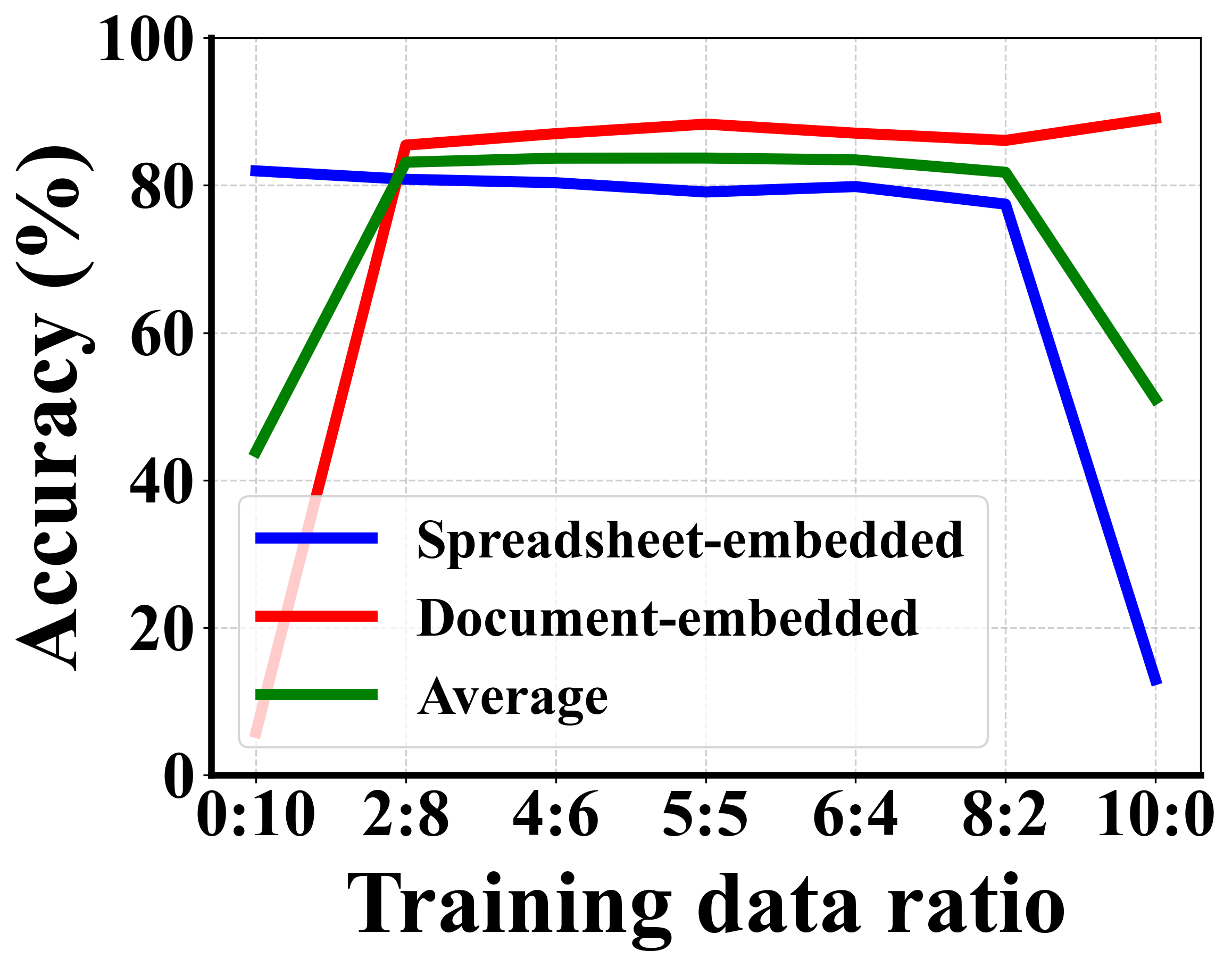}
	}		

 \caption{\label{fig:trainingstrategy} Effects of data size and ratio.}
\end{figure}


\section{Conclusion}
\label{sec:conclusion}
This paper introduces \smodel(8B) tailored for tabular data manipulation in real office scenarios. We gather actual requirements from office settings and identify document-embedded and spreadsheet-embedded scenarios. Ensuring high-quality data through extended reasoning processes and cross-way validation on automatically generated training data, the resulting \smodel performs comparably to GPT-4o and surpasses it in the spreadsheet-embedded scenario. We anticipate that our dataset, model checkpoint, and code will provide a cost-effective solution for enhancing LLM capabilities for tables.


\section*{Limitations}
\label{sec:limitations}
\textbf{Diversity of Generated Data.} The training data, although extensive, may not fully capture the diversity of tabular data encountered in real-world applications. In the question generation phase of the cross-way strategy, the questions generated by GLM-4-Public are based on tabular data from the existing benchmark, which may not cover all the cases, and some of the generated questions are not challenging enough to enhance the model's reasoning about the tabular data. Therefore, the reliance on existing benchmarks and automatically generated data could introduce biases or gaps in the model's understanding of less common or more complex table structures. Future research could further improve the quality of synthetic training data by categorizing the difficulty of generating problems.


\bibliography{custom}

\begin{thebibliography}{43}
\providecommand{\natexlab}[1]{#1}

\bibitem[{Badaro et~al.(2023)Badaro, Saeed et~al.}]{badaro2023Transformers}
Gilbert Badaro, Mohammed Saeed, et~al. 2023.
\newblock Transformers for tabular data representation: {{A}} survey of models and applications.
\newblock \emph{TACL}.

\bibitem[{Bi et~al.(2024)Bi, Chen, Chen, Chen, Dai, Deng, Ding, Dong, Du, Fu et~al.}]{bi2024deepseek}
Xiao Bi, Deli Chen, Guanting Chen, Shanhuang Chen, Damai Dai, Chengqi Deng, Honghui Ding, Kai Dong, Qiushi Du, Zhe Fu, et~al. 2024.
\newblock Deepseek llm: Scaling open-source language models with longtermism.
\newblock \emph{arXiv preprint arXiv:2401.02954}.

\bibitem[{Cheng et~al.(2023)Cheng, Xie, Shi, Li, Nadkarni, Hu, Xiong, Radev, Ostendorf, Zettlemoyer, Smith, and Yu}]{Binder}
Zhoujun Cheng, Tianbao Xie, Peng Shi, Chengzu Li, Rahul Nadkarni, Yushi Hu, Caiming Xiong, Dragomir Radev, Mari Ostendorf, Luke Zettlemoyer, Noah~A. Smith, and Tao Yu. 2023.
\newblock Binding language models in symbolic languages.
\newblock \emph{ICLR}.

\bibitem[{Dietterich et~al.(2002)}]{dietterich2002ensemble}
Thomas~G Dietterich et~al. 2002.
\newblock Ensemble learning.
\newblock \emph{The handbook of brain theory and neural networks}, 2(1):110--125.

\bibitem[{Dong et~al.(2022)Dong, Cheng et~al.}]{dong2022Table}
Haoyu Dong, Zhoujun Cheng, et~al. 2022.
\newblock Table {{Pre-training}}: {{A Survey}} on {{Model Architectures}}, {{Pre-training Objectives}}, and {{Downstream Tasks}}.
\newblock In \emph{IJCAI}.

\bibitem[{Dong et~al.(2023)Dong, Zhang, Ge, Mao, Gao, Chen, Lin, and Lou}]{c3}
Xuemei Dong, Chao Zhang, Yuhang Ge, Yuren Mao, Yunjun Gao, Lu~Chen, Jinshu Lin, and Dongfang Lou. 2023.
\newblock \href {https://arxiv.org/abs/2307.07306} {{C3:} zero-shot text-to-sql with chatgpt}.
\newblock \emph{CoRR}, abs/2307.07306.

\bibitem[{Dubey et~al.(2024)Dubey, Jauhri, Pandey, Kadian, Al-Dahle, Letman, Mathur, Schelten, Yang, Fan et~al.}]{llama3-report}
Abhimanyu Dubey, Abhinav Jauhri, Abhinav Pandey, Abhishek Kadian, Ahmad Al-Dahle, Aiesha Letman, Akhil Mathur, Alan Schelten, Amy Yang, Angela Fan, et~al. 2024.
\newblock \href {https://arxiv.org/abs/2407.21783} {The llama 3 herd of models}.
\newblock \emph{arXiv preprint arXiv:2407.21783}.

\bibitem[{GLM et~al.(2024)GLM, Zeng, Xu, Wang, Zhang, Yin, Zhang, Rojas, Feng, Zhao et~al.}]{glm2024chatglm}
Team GLM, Aohan Zeng, Bin Xu, Bowen Wang, Chenhui Zhang, Da~Yin, Dan Zhang, Diego Rojas, Guanyu Feng, Hanlin Zhao, et~al. 2024.
\newblock Chatglm: A family of large language models from glm-130b to glm-4 all tools.
\newblock \emph{arXiv preprint arXiv:2406.12793}.

\bibitem[{Gong et~al.(2020)Gong, Sun, Feng, Qin, Bi, Liu, and Liu}]{tablegpt2020}
Heng Gong, Yawei Sun, Xiaocheng Feng, Bing Qin, Wei Bi, Xiaojiang Liu, and Ting Liu. 2020.
\newblock \href {https://doi.org/10.18653/v1/2020.coling-main.179} {{T}able{GPT}: Few-shot table-to-text generation with table structure reconstruction and content matching}.
\newblock In \emph{Proceedings of the 28th International Conference on Computational Linguistics}, pages 1978--1988, Barcelona, Spain (Online). International Committee on Computational Linguistics.

\bibitem[{Herzig et~al.(2020)Herzig, Nowak, M{\"u}ller, Piccinno, and Eisenschlos}]{tapas}
Jonathan Herzig, Pawel~Krzysztof Nowak, Thomas M{\"u}ller, Francesco Piccinno, and Julian Eisenschlos. 2020.
\newblock \href {https://doi.org/10.18653/v1/2020.acl-main.398} {{T}a{P}as: Weakly supervised table parsing via pre-training}.
\newblock In \emph{Proceedings of the 58th Annual Meeting of the Association for Computational Linguistics}, pages 4320--4333, Online. Association for Computational Linguistics.

\bibitem[{Hu et~al.(2023)Hu, Fu, Du, Luo, Zhao, and Zhao}]{hu2023chatdb}
Chenxu Hu, Jie Fu, Chenzhuang Du, Simian Luo, Junbo Zhao, and Hang Zhao. 2023.
\newblock Chatdb: Augmenting llms with databases as their symbolic memory.
\newblock \emph{arXiv preprint arXiv:2306.03901}.

\bibitem[{Hu et~al.(2024)Hu, Zhao, Wei, Chai, Wang, Wang, Su, Xu, Zhu, Cheng, Yuan, Kuang, Yang, Yang, and Wu}]{infiagent}
Xueyu Hu, Ziyu Zhao, Shuang Wei, Ziwei Chai, Guoyin Wang, Xuwu Wang, Jing Su, Jingjing Xu, Ming Zhu, Yao Cheng, Jianbo Yuan, Kun Kuang, Yang Yang, Hongxia Yang, and Fei Wu. 2024.
\newblock \href {https://arxiv.org/abs/2401.05507} {Infiagent-dabench: Evaluating agents on data analysis tasks}.
\newblock \emph{Preprint}, arXiv:2401.05507.

\bibitem[{Hulsebos et~al.(2023)Hulsebos, Demiralp, and Groth}]{hulsebos2023gittables}
Madelon Hulsebos, {\c{C}}agatay Demiralp, and Paul Groth. 2023.
\newblock Gittables: A large-scale corpus of relational tables.
\newblock \emph{Proceedings of the ACM on Management of Data}, 1(1):1--17.

\bibitem[{Iida et~al.(2021)Iida, Thai, Manjunatha, and Iyyer}]{iida2021tabbie}
Hiroshi Iida, Dung Thai, Varun Manjunatha, and Mohit Iyyer. 2021.
\newblock Tabbie: Pretrained representations of tabular data.
\newblock \emph{arXiv preprint arXiv:2105.02584}.

\bibitem[{Jiang et~al.(2023)Jiang, Zhou, Dong, Ye, Zhao, and Wen}]{structgpt}
Jinhao Jiang, Kun Zhou, Zican Dong, Keming Ye, Xin Zhao, and Ji-Rong Wen. 2023.
\newblock \href {https://doi.org/10.18653/v1/2023.emnlp-main.574} {{S}truct{GPT}: A general framework for large language model to reason over structured data}.
\newblock In \emph{Proceedings of the 2023 Conference on Empirical Methods in Natural Language Processing}, pages 9237--9251, Singapore. Association for Computational Linguistics.

\bibitem[{Ke et~al.(2023)Ke, Wen, Feng, Liu, Lei, Cheng, Wang, Zeng, Dong, Wang, Tang, and Huang}]{critiquellm}
Pei Ke, Bosi Wen, Zhuoer Feng, Xiao Liu, Xuanyu Lei, Jiale Cheng, Shengyuan Wang, Aohan Zeng, Yuxiao Dong, Hongning Wang, Jie Tang, and Minlie Huang. 2023.
\newblock \href {https://arxiv.org/abs/2311.18702} {Critiquellm: Scaling llm-as-critic for effective and explainable evaluation of large language model generation}.
\newblock \emph{Preprint}, arXiv:2311.18702.

\bibitem[{Li et~al.(2023{\natexlab{a}})Li, Su, Chen, Li, and Zhang}]{sheetcopilot}
Hongxin Li, Jingran Su, Yuntao Chen, Qing Li, and Zhaoxiang Zhang. 2023{\natexlab{a}}.
\newblock \href {https://openreview.net/forum?id=tfyr2zRVoK} {Sheetcopilot: Bringing software productivity to the next level through large language models}.
\newblock In \emph{Thirty-seventh Conference on Neural Information Processing Systems}.

\bibitem[{Li et~al.(2023{\natexlab{b}})Li, He, Yashar, Cui, Ge, Zhang, Fainman, Zhang, and Chaudhuri}]{li2023tablegpt}
Peng Li, Yeye He, Dror Yashar, Weiwei Cui, Song Ge, Haidong Zhang, Danielle~Rifinski Fainman, Dongmei Zhang, and Surajit Chaudhuri. 2023{\natexlab{b}}.
\newblock \href {https://arxiv.org/abs/2310.09263} {Table-gpt: Table-tuned gpt for diverse table tasks}.
\newblock \emph{Preprint}, arXiv:2310.09263.

\bibitem[{Liang et~al.(2023)Liang, Bommasani, Lee, Tsipras, Soylu, Yasunaga, Zhang, Narayanan, Wu, Kumar, Newman, Yuan, Yan, Zhang, Cosgrove, Manning, Ré, Acosta-Navas, Hudson, Zelikman, Durmus, Ladhak, Rong, Ren, Yao, Wang, Santhanam, Orr, Zheng, Yuksekgonul, Suzgun, Kim, Guha, Chatterji, Khattab, Henderson, Huang, Chi, Xie, Santurkar, Ganguli, Hashimoto, Icard, Zhang, Chaudhary, Wang, Li, Mai, Zhang, and Koreeda}]{liang2023holistic}
Percy Liang, Rishi Bommasani, Tony Lee, Dimitris Tsipras, Dilara Soylu, Michihiro Yasunaga, Yian Zhang, Deepak Narayanan, Yuhuai Wu, Ananya Kumar, Benjamin Newman, Binhang Yuan, Bobby Yan, Ce~Zhang, Christian Cosgrove, Christopher~D. Manning, Christopher Ré, Diana Acosta-Navas, Drew~A. Hudson, Eric Zelikman, Esin Durmus, Faisal Ladhak, Frieda Rong, Hongyu Ren, Huaxiu Yao, Jue Wang, Keshav Santhanam, Laurel Orr, Lucia Zheng, Mert Yuksekgonul, Mirac Suzgun, Nathan Kim, Neel Guha, Niladri Chatterji, Omar Khattab, Peter Henderson, Qian Huang, Ryan Chi, Sang~Michael Xie, Shibani Santurkar, Surya Ganguli, Tatsunori Hashimoto, Thomas Icard, Tianyi Zhang, Vishrav Chaudhary, William Wang, Xuechen Li, Yifan Mai, Yuhui Zhang, and Yuta Koreeda. 2023.
\newblock \href {https://arxiv.org/abs/2211.09110} {Holistic evaluation of language models}.
\newblock \emph{Preprint}, arXiv:2211.09110.

\bibitem[{Lin(2004)}]{lin2004rouge}
Chin-Yew Lin. 2004.
\newblock Rouge: A package for automatic evaluation of summaries.
\newblock In \emph{Text summarization branches out}, pages 74--81.

\bibitem[{Liu et~al.(2024)Liu, Feng, Xue, Wang, Wu, Lu, Zhao, Deng, Zhang, Ruan et~al.}]{liu2024deepseek}
Aixin Liu, Bei Feng, Bing Xue, Bingxuan Wang, Bochao Wu, Chengda Lu, Chenggang Zhao, Chengqi Deng, Chenyu Zhang, Chong Ruan, et~al. 2024.
\newblock Deepseek-v3 technical report.
\newblock \emph{arXiv preprint arXiv:2412.19437}.

\bibitem[{Liu et~al.(2022)Liu, Chen, Guo, Ziyadi, Lin, Chen, and Lou}]{tapex}
Qian Liu, Bei Chen, Jiaqi Guo, Morteza Ziyadi, Zeqi Lin, Weizhu Chen, and Jian-Guang Lou. 2022.
\newblock \href {https://openreview.net/forum?id=O50443AsCP} {{TAPEX}: Table pre-training via learning a neural {SQL} executor}.
\newblock In \emph{International Conference on Learning Representations}.

\bibitem[{Nan et~al.(2022)Nan, Hsieh, Mao, Lin, Verma, Zhang, Kryściński, Schoelkopf, Kong, Tang, Mutuma, Rosand, Trindade, Bandaru, Cunningham, Xiong, and Radev}]{Nan2021FeTaQAFT}
Linyong Nan, Chiachun Hsieh, Ziming Mao, Xi~Victoria Lin, Neha Verma, Rui Zhang, Wojciech Kryściński, Hailey Schoelkopf, Riley Kong, Xiangru Tang, Mutethia Mutuma, Ben Rosand, Isabel Trindade, Renusree Bandaru, Jacob Cunningham, Caiming Xiong, and Dragomir Radev. 2022.
\newblock Fetaqa: Free-form table question answering.
\newblock \emph{Transactions of the Association for Computational Linguistics}, 10:35--49.

\bibitem[{OpenAI(2023)}]{gpt4}
OpenAI. 2023.
\newblock \href {https://doi.org/10.48550/arXiv.2303.08774} {{GPT-4} technical report}.
\newblock \emph{CoRR}, abs/2303.08774.

\bibitem[{Ouyang et~al.(2022)Ouyang, Wu, Jiang, Almeida, Wainwright, Mishkin, Zhang, Agarwal, Slama, Ray, Schulman, Hilton, Kelton, Miller, Simens, Askell, Welinder, Christiano, Leike, and Lowe}]{InstructGPT}
Long Ouyang, Jeffrey Wu, Xu~Jiang, Diogo Almeida, Carroll~L. Wainwright, Pamela Mishkin, Chong Zhang, Sandhini Agarwal, Katarina Slama, Alex Ray, John Schulman, Jacob Hilton, Fraser Kelton, Luke Miller, Maddie Simens, Amanda Askell, Peter Welinder, Paul~F. Christiano, Jan Leike, and Ryan Lowe. 2022.
\newblock \href {http://papers.nips.cc/paper\_files/paper/2022/hash/b1efde53be364a73914f58805a001731-Abstract-Conference.html} {Training language models to follow instructions with human feedback}.
\newblock In \emph{NeurIPS}.

\bibitem[{Pasupat and Liang(2015)}]{WTQ}
Panupong Pasupat and Percy Liang. 2015.
\newblock \href {https://doi.org/10.3115/v1/P15-1142} {Compositional semantic parsing on semi-structured tables}.
\newblock In \emph{Proceedings of the 53rd Annual Meeting of the Association for Computational Linguistics and the 7th International Joint Conference on Natural Language Processing (Volume 1: Long Papers)}, pages 1470--1480, Beijing, China. Association for Computational Linguistics.

\bibitem[{Pourreza and Rafiei(2023)}]{dinsql}
Mohammadreza Pourreza and Davood Rafiei. 2023.
\newblock Din-sql: Decomposed in-context learning of text-to-sql with self-correction.
\newblock \emph{arXiv preprint arXiv:2304.11015}.

\bibitem[{Rozi{\`e}re et~al.(2023)Rozi{\`e}re, Gehring, Gloeckle, Sootla, Gat, Tan, Adi, Liu, Remez, Rapin, Kozhevnikov, Evtimov, Bitton, Bhatt, Ferrer, Grattafiori, Xiong, D'efossez, Copet, Azhar, Touvron, Martin, Usunier, Scialom, and Synnaeve}]{CodeLlama}
Baptiste Rozi{\`e}re, Jonas Gehring, Fabian Gloeckle, Sten Sootla, Itai Gat, Xiaoqing Tan, Yossi Adi, Jingyu Liu, Tal Remez, J{\'e}r{\'e}my Rapin, Artyom Kozhevnikov, I.~Evtimov, Joanna Bitton, Manish~P Bhatt, Cristian~Cant{\'o}n Ferrer, Aaron Grattafiori, Wenhan Xiong, Alexandre D'efossez, Jade Copet, Faisal Azhar, Hugo Touvron, Louis Martin, Nicolas Usunier, Thomas Scialom, and Gabriel Synnaeve. 2023.
\newblock \href {https://api.semanticscholar.org/CorpusID:261100919} {Code llama: Open foundation models for code}.
\newblock \emph{ArXiv}, abs/2308.12950.

\bibitem[{Su et~al.(2024)Su, Wang, Ye, Zhou, Zhang, Zhu, Wang, Xu, Chen, Li et~al.}]{su2024tablegpt2}
Aofeng Su, Aowen Wang, Chao Ye, Chen Zhou, Ga~Zhang, Guangcheng Zhu, Haobo Wang, Haokai Xu, Hao Chen, Haoze Li, et~al. 2024.
\newblock Tablegpt2: A large multimodal model with tabular data integration.
\newblock \emph{arXiv preprint arXiv:2411.02059}.

\bibitem[{Wei et~al.(2022)Wei, Wang, Schuurmans, Bosma, Ichter, Xia, Chi, Le, and Zhou}]{cot}
Jason Wei, Xuezhi Wang, Dale Schuurmans, Maarten Bosma, Brian Ichter, Fei Xia, Ed~H. Chi, Quoc~V. Le, and Denny Zhou. 2022.
\newblock \href {http://papers.nips.cc/paper\_files/paper/2022/hash/9d5609613524ecf4f15af0f7b31abca4-Abstract-Conference.html} {Chain-of-thought prompting elicits reasoning in large language models}.
\newblock In \emph{NeurIPS}.

\bibitem[{Xie et~al.(2022)Xie, Wu, Shi, Zhong, Scholak, Yasunaga, Wu, Zhong, Yin, Wang, Zhong, Wang, Li, Boyle, Ni, Yao, Radev, Xiong, Kong, Zhang, Smith, Zettlemoyer, and Yu}]{unifiedskg}
Tianbao Xie, Chen~Henry Wu, Peng Shi, Ruiqi Zhong, Torsten Scholak, Michihiro Yasunaga, Chien-Sheng Wu, Ming Zhong, Pengcheng Yin, Sida~I. Wang, Victor Zhong, Bailin Wang, Chengzu Li, Connor Boyle, Ansong Ni, Ziyu Yao, Dragomir Radev, Caiming Xiong, Lingpeng Kong, Rui Zhang, Noah~A. Smith, Luke Zettlemoyer, and Tao Yu. 2022.
\newblock \href {https://doi.org/10.18653/v1/2022.emnlp-main.39} {{U}nified{SKG}: Unifying and multi-tasking structured knowledge grounding with text-to-text language models}.
\newblock In \emph{Proceedings of the 2022 Conference on Empirical Methods in Natural Language Processing}, pages 602--631, Abu Dhabi, United Arab Emirates. Association for Computational Linguistics.

\bibitem[{Xue et~al.(2023)Xue, Jiang, Shi, Cheng, Chen, Yang, Zhang, He, Zhang, Wei et~al.}]{xue2023db}
Siqiao Xue, Caigao Jiang, Wenhui Shi, Fangyin Cheng, Keting Chen, Hongjun Yang, Zhiping Zhang, Jianshan He, Hongyang Zhang, Ganglin Wei, et~al. 2023.
\newblock Db-gpt: Empowering database interactions with private large language models.
\newblock \emph{arXiv preprint arXiv:2312.17449}.

\bibitem[{Ye et~al.(2023)Ye, Hui, Yang, Li, Huang, and Li}]{DATER}
Yunhu Ye, Binyuan Hui, Min Yang, Binhua Li, Fei Huang, and Yongbin Li. 2023.
\newblock \href {https://doi.org/10.1145/3539618.3591708} {Large language models are versatile decomposers: Decomposing evidence and questions for table-based reasoning}.
\newblock In \emph{Proceedings of the 46th International ACM SIGIR Conference on Research and Development in Information Retrieval}, SIGIR '23, page 174–184, New York, NY, USA. Association for Computing Machinery.

\bibitem[{Yin et~al.(2020)Yin, Neubig, Yih, and Riedel}]{tabert}
Pengcheng Yin, Graham Neubig, Wen-tau Yih, and Sebastian Riedel. 2020.
\newblock \href {https://doi.org/10.18653/v1/2020.acl-main.745} {{T}a{BERT}: Pretraining for joint understanding of textual and tabular data}.
\newblock In \emph{Proceedings of the 58th Annual Meeting of the Association for Computational Linguistics}, pages 8413--8426, Online. Association for Computational Linguistics.

\bibitem[{Yu et~al.(2020)Yu, Wu, Lin, Wang, Tan, Yang, Radev, Socher, and Xiong}]{yu2020grappa}
Tao Yu, Chien-Sheng Wu, Xi~Victoria Lin, Bailin Wang, Yi~Chern Tan, Xinyi Yang, Dragomir Radev, Richard Socher, and Caiming Xiong. 2020.
\newblock Grappa: Grammar-augmented pre-training for table semantic parsing.
\newblock \emph{arXiv preprint arXiv:2009.13845}.

\bibitem[{Yu et~al.(2018)Yu, Zhang, Yang, Yasunaga, Wang, Li, Ma, Li, Yao, Roman, Zhang, and Radev}]{Spider}
Tao Yu, Rui Zhang, Kai Yang, Michihiro Yasunaga, Dongxu Wang, Zifan Li, James Ma, Irene Li, Qingning Yao, Shanelle Roman, Zilin Zhang, and Dragomir Radev. 2018.
\newblock Spider: A large-scale human-labeled dataset for complex and cross-domain semantic parsing and text-to-sql task.
\newblock In \emph{Proceedings of the 2018 Conference on Empirical Methods in Natural Language Processing}, Brussels, Belgium. Association for Computational Linguistics.

\bibitem[{Zhang et~al.(2023{\natexlab{a}})Zhang, Yue, Li, and Sun}]{TableLlama}
Tianshu Zhang, Xiang Yue, Yifei Li, and Huan Sun. 2023{\natexlab{a}}.
\newblock \href {https://arxiv.org/abs/2311.09206} {Tablellama: Towards open large generalist models for tables}.
\newblock \emph{Preprint}, arXiv:2311.09206.

\bibitem[{Zhang et~al.(2023{\natexlab{b}})Zhang, Shen, Lu, and Zhuang}]{datacopilot}
Wenqi Zhang, Yongliang Shen, Weiming Lu, and Yueting Zhuang. 2023{\natexlab{b}}.
\newblock Data-copilot: Bridging billions of data and humans with autonomous workflow.
\newblock \emph{arXiv preprint arXiv:2306.07209}.

\bibitem[{Zhang et~al.(2023{\natexlab{c}})Zhang, Henkel, Floratou, Cahoon, Deep, and Patel}]{reactable}
Yunjia Zhang, Jordan Henkel, Avrilia Floratou, Joyce Cahoon, Shaleen Deep, and Jignesh~M. Patel. 2023{\natexlab{c}}.
\newblock \href {https://arxiv.org/abs/2310.00815} {Reactable: Enhancing react for table question answering}.
\newblock \emph{Preprint}, arXiv:2310.00815.

\bibitem[{Zhao et~al.(2023)Zhao, Liu, Wan, Wang, Deng, and Yu}]{zhao2023localize}
Wenting Zhao, Ye~Liu, Yao Wan, Yibo Wang, Zhongfen Deng, and Philip~S Yu. 2023.
\newblock Localize, retrieve and fuse: A generalized framework for free-form question answering over tables.
\newblock \emph{arXiv preprint arXiv:2309.11049}.

\bibitem[{Zhong et~al.(2017)Zhong, Xiong, and Socher}]{zhongSeq2SQL2017}
Victor Zhong, Caiming Xiong, and Richard Socher. 2017.
\newblock Seq2sql: Generating structured queries from natural language using reinforcement learning.
\newblock \emph{CoRR}, abs/1709.00103.

\bibitem[{Zhu et~al.(2021)Zhu, Lei, Huang, Wang, Zhang, Lv, Feng, and Chua}]{TAT-QA}
Fengbin Zhu, Wenqiang Lei, Youcheng Huang, Chao Wang, Shuo Zhang, Jiancheng Lv, Fuli Feng, and Tat-Seng Chua. 2021.
\newblock \href {https://doi.org/10.18653/v1/2021.acl-long.254} {{TAT}-{QA}: A question answering benchmark on a hybrid of tabular and textual content in finance}.
\newblock In \emph{Proceedings of the 59th Annual Meeting of the Association for Computational Linguistics and the 11th International Joint Conference on Natural Language Processing (Volume 1: Long Papers)}, pages 3277--3287, Online. Association for Computational Linguistics.

\bibitem[{Zhu et~al.(2024)Zhu, Liu, Feng, Wang, Li, and Chua}]{zhu2024tat}
Fengbin Zhu, Ziyang Liu, Fuli Feng, Chao Wang, Moxin Li, and Tat-Seng Chua. 2024.
\newblock Tat-llm: A specialized language model for discrete reasoning over tabular and textual data.
\newblock \emph{arXiv preprint arXiv:2401.13223}.

\end{thebibliography}

\clearpage

\appendix
\label{sec:appendix}
\section{License}
\label{sec:license}
We utilize training datasets from WikiTQ, FeTaQA, and TAT-QA for the document-embedded scenario, alongside WikiSQL, Spider, and GitTables for the spreadsheet-embedded scenario.

WikiTQ, FeTaQA, and Spider are distributed under the CC BY-SA-4.0 license, which allows users to freely share and adapt the material, provided appropriate credit is given and any derivative works are distributed under the same license. TAT-QA is distributed under the MIT license, permitting unrestricted use, modification, and distribution, provided the original license terms are included. WikiSQL is distributed under the BSD-3-Clause license, which allows for free use, modification, and distribution, with the condition that the original copyright notice and license terms are preserved. GitTables is distributed under the CC0-1.0 license, which effectively places the data in the public domain, allowing for unrestricted use, modification, and distribution without any requirements.

The parameters of \model, as well as the associated training data, will be made publicly accessible. In adherence to open-access principles, both the training data and the fine-tuned model will be released under the CC BY-SA 4.0 license. This approach supports the ethos of open science and fosters the reuse and redistribution of our resources under consistent terms.

\section{Web Application}
\label{sec:webapp}

We launch our \smodel as a web application\footnote{\url{https://tablellm.github.io/}}, with a screenshot shown in Figure~\ref{fig:app}, where Figure~\ref{subfig:update} shows an instruction with the update operation and Figure~\ref{subfig:chart} shows an instruction with the chart operatio. The typical workflow is as follows: Users begin by uploading their tabular data embedded in documents (with support for Word and PDF formats) and spreadsheets (supporting Excel and CSV formats). The system utilizes Grobid\footnote{https://github.com/kermitt2/grobid} to parse PDF files and python-docx\footnote{https://pypi.org/project/python-docx/} for Word files, converting them into CSV format for web visualization. Users then enter queries or instructions in the query box. Depending on the type of uploaded document or spreadsheet, appropriate prompts from Figure~\ref{fig:overview} guide the \smodel to generate answers. The response could be a table, a chart, or a textual answer. Additionally, the application offers a feature for merging two tables, where users can upload two spreadsheets and specify the merging conditions in the query box\footnote{Currently, the system is configured to support the merging of two tables only.}.

We open the application for trial to a diverse group including teachers, students, administrators from universities, marketing professionals, human resources personnel, and research and development specialists. They are encouraged to provide feedback by clicking ``Thumbs up'' or ``Thumbs down''. So far, we have collected 2,000 use cases from users, with 1,560 involving spreadsheet-embedded scenarios (1,869 for single table operations and 131 for double table operations) and 440 for document-embedded scenarios. Among these, we have received 1,473 feedbacks with 1,293 ``Thumbs up'' and 180 ``Thumbs down'', closely aligning with the performance metrics reported in Table~\ref{tb:overall}. We conduct an error analysis in Appendix~\ref{sec:error} for further improvement.

\begin{figure*}[t]
	\centering
	\subfigure[Update operation]{\label{subfig:update}
		\includegraphics[width=0.9\textwidth]{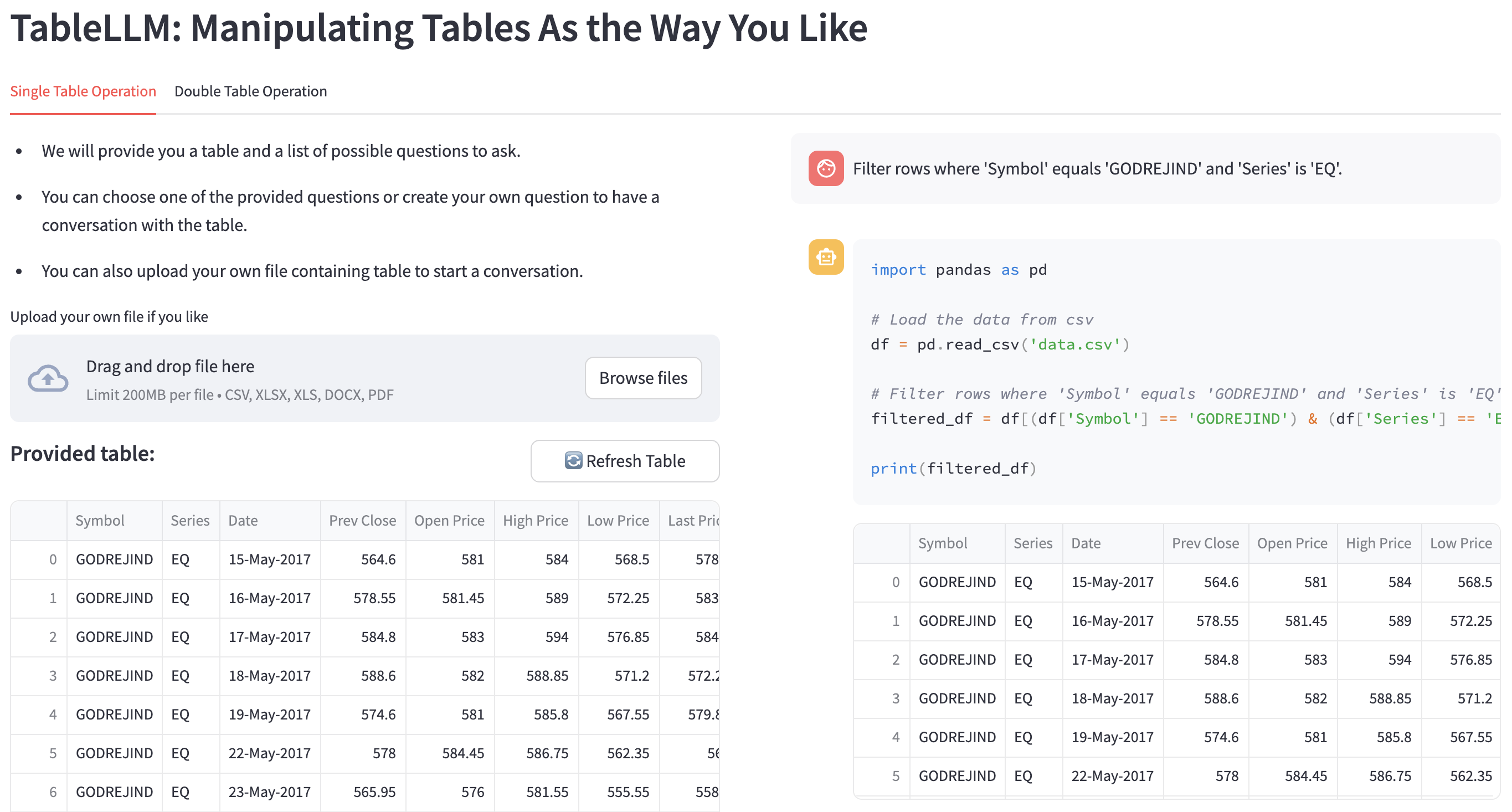}
	}
	\subfigure[Chat operation]{\label{subfig:chart}
		\includegraphics[width=0.9\textwidth]{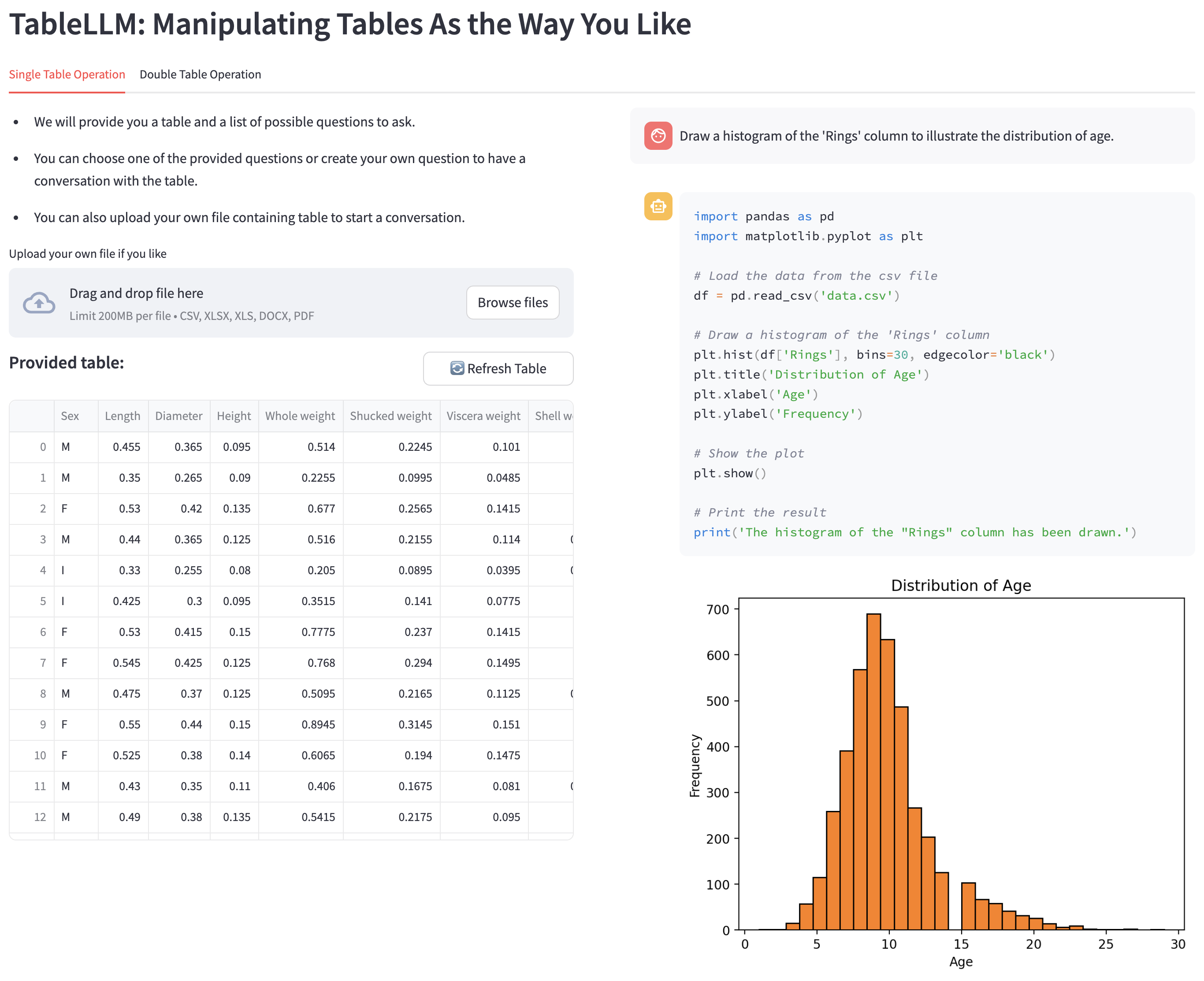}
	}		
 \caption{\label{fig:app} The web application deployed with our \model.}
\end{figure*}

\section{The Survey Details}
\label{sec:surveyexplanation}

We conduct a survey among universities and enterprises to assess users' needs for tabular data-related tasks in real office environments. We obtain a total of 507 valid responses, representing various roles, including research and development specialists (36.69\%), teachers (14.40\%), administrators (14.00\%), students (12.62\%), marketing professionals (4.14\%), and human resources personnel (1.97\%).

The survey results, depicted in Figure~\ref{fig:survey}, reveal: (a) Participants' demands for various table-related tasks, with TableQA (80.28\%) being the most sought-after, followed by Table revision (71.40\%), which involves creating, updating, and deleting tables. Tasks with demand exceeding 200 include TableQA, Table revision, Chart creation, Table matching, Duplicate data removal, Error detection, Missing value detection, and Table extraction. Notably, table extraction, focused on format conversion, is in demand but can be handled efficiently by non-LLM tools, thus not considered LLM-related tasks. (b) There's relatively less demand for tasks like column type annotation, entity linking, and fact verification. (c) Participants prefer Excel, Word, PDF, and CSV formats, and (d) long tables with more than 50 rows.

Below is the complete survey on table usage:

\global\mdfdefinestyle{exampledefault}{%
    linecolor=black,
    frametitlerule = true
}
\begin{enumerate}

    \item What is your occupation?
    \begin{itemize}
        \item[\textit{A}] Student
        \item[\textit{B}] Teacher
        \item[\textit{C}] Administrator
        \item[\textit{D}] Human Resources Professional
        \item[\textit{E}] Marketing Professional
        \item[\textit{F}] Research and Development Specialist
        \item[\textit{G}] Others[Fill in the Blank]
    \end{itemize}
    
    \item In your daily work, how often do you work with tables (such as Excel, CSV, or direct access to databases)?
    \begin{itemize}
        \item[\textit{A}] Rarely use (less than once a day on average)
        \item[\textit{B}] Occasionally use (1 to 5 times per day)
        \item[\textit{C}] Frequently use (5 to 20 times per day)
        \item[\textit{D}] Is my work theme (use more than 20 times a day)
    \end{itemize}
    
    \item In your normal work, What are the sizes of tables that you typically work with?[Multiple choice question]
    \begin{itemize}
        \item[\textit{A}] Tables under 50 rows.
        \item[\textit{B}] Tables over 50 rows.
    \end{itemize}
    
    \item What types of tables do you typically encounter and handle in your daily work?[Multiple choice question]
    \begin{itemize}
        \item[\textit{A}] Excel
        \item[\textit{B}] Word
        \item[\textit{C}] HTML
        \item[\textit{D}] CSV
        \item[\textit{E}] PDF
        \item[\textit{F}] Markdown
        \item[\textit{G}] Others[Fill in the Blank]
    \end{itemize}
    
    
    \item Which Table manipulation tasks do you need to use in your work?[Multiple choice question]
    \begin{itemize}
        \item[\textit{A}] TableQA, e.g.,
        \begin{itemize}
            \item Find the number of people with grade above 90;
            \item Group them according to 90-100 points, 80-90 points, 60-80 points and 60 points, and count the number of people in each score segment;
            \item Find all conferences held in Jiangsu in the second half of 2023;
        \end{itemize}
        \item[\textit{B}] Table revision, e.g.,
        \begin{itemize}
            \item Sort by height column;
            \item Convert the Date column to Month/Day/year format;
            \item Insert a column ``total score'', representing the weighted sum of 60\% and 40\% from the first to the third column;
            \item Delete the ``normal score'' column;
        \end{itemize}
        \item[\textit{C}] Chart, e.g.,
        \begin{itemize}
            \item Draw statistical drawings, such as line diagrams, column charts, pie charts;
        \end{itemize}
        \item[\textit{D}] None.
    \end{itemize}
    
    
    \item Which Table cleaning tasks do you need in your work?[Multiple choice question]
    \begin{itemize}
        \item[\textit{A}] Missing value detection, e.g.,
        \begin{itemize}
            \item Detect missing values and fill in the mean value of the corresponding column;
        \end{itemize}
        \item[\textit{B}] Error detection, e.g.,
        \begin{itemize}
            \item Check a cell whose format is not ``month/day/year'' and convert it;
        \end{itemize}
        \item[\textit{C}] Delete duplicate data, e.g.,
        \begin{itemize}
            \item To filter duplicate data by name, only keep the first row with the same name;
        \end{itemize}
        \item[\textit{D}] None.
    \end{itemize}


    \item Which Table analysis tasks do you need in your work?[Multiple choice question]
    \begin{itemize}
        \item[\textit{A}] Column type annotation, e.g.,
        \begin{itemize}
            \item Given some examples of a column like 1,000 RMB, 1,500 RMB, and 2,000 RMB, name the column;
        \end{itemize}
        \item[\textit{B}] Entity linking, e.g.,
        \begin{itemize}
            \item Given a candidate combination of column names, assign an appropriate column name to each column in the table;
        \end{itemize}
        \item[\textit{C}] Row-to-row transform, e.g.,
        \begin{itemize}
            \item Predict the rating of the fourth team based on the ``win-loss rating'' of the first three teams;
        \end{itemize}
        \item[\textit{D}] Fact verification, e.g.,
        \begin{itemize}
            \item Based on the content of the table, determine whether ``profit growth in the first quarter of 2023 is 10\%'' is true;
        \end{itemize}
        \item[\textit{E}] None.
    \end{itemize}


    \item Which Table-to-Text tasks do you need in your work?[Multiple choice question]
    \begin{itemize}
        \item[\textit{A}] Summarization, e.g.,
        \begin{itemize}
            \item Generate a title for the table;
        \end{itemize}
        \item[\textit{B}] Dialogue generation, e.g.,
        \begin{itemize}
            \item Given the table and the history of the conversation, generate the next conversation;
        \end{itemize}
        \item[\textit{C}] None.
    \end{itemize}


    \item Which Table augmentation tasks do you need in your work? [Multiple choice question]
    \begin{itemize}
        \item[\textit{A}] Row population, e.g.,
        \begin{itemize}
            \item Given ``name'', ``age'', ``height'', generate specific row data;
        \end{itemize}
        \item[\textit{B}] Schema augmentation, e.g.,
        \begin{itemize}
            \item Given ``Date'', ``Growth rate'', ``Net income'', expand the other columns;
        \end{itemize}
        \item[\textit{C}] None.
    \end{itemize}
    
    \item Do you need Table matching at work? (For example, merge two tables as required)
    \begin{itemize}
        \item[\textit{A}] Yes, I need Table matching.
        \item[\textit{B}] No, I don't need Table matching.
    \end{itemize}

    \item Do you need Table extraction at work? (For example, organize the Markdown format table into Excel, extract a table from the web page to Excel)
    \begin{itemize}
        \item[\textit{A}] Yes, I need Table extraction.
        \item[\textit{B}] No, I don't need Table extraction.
    \end{itemize}

    \item Do you have tabular tasks that do not fit into the above categories? If yes, please give an example.[Fill in the Blank]

\end{enumerate}

\section{Verification of Cross-way Validation}
\label{sec:cross-way-verification}

We use $Y_a$ to denote that the first response $A$ is correct, $Y_b$ to denote that the second response $B$ is correct, $Y$ to denote that both responses are correct, and $E$ to denote that the two responses are consistent. We will prove the following: \\
\begin{theorem}\label{thm:consistency_check}
If $A$ and $B$ come from the same distribution $D$, $P(Y_a) = P(Y_b) = p > 1/2$, then the consistency check is better than single inference, that is, $P(Y|E) \geq P(Y_a)$. \\
\end{theorem}
\begin{theorem}\label{thm:independent_improvement}
If  \( P(Y_a) = P(Y_b) = p \),  \( A \) and \( B \) are sampled from independent distributions \( D_A \) and \( D_B \) respectively, the outcome will improve (in terms of expected value), 
\beal{&E[P(Y|E) | Y_a \sim D, Y_b \sim D] \nonumber \\
&\leq E[P(Y|E) | Y_a \sim D_A, Y_b \sim D_B]. \nonumber}
\end{theorem}

\begin{lemma}
If  $ \frac{1}{2} \leq p \leq 1 $, then $ \frac{p^2}{p^2 + (1 - p)^2} \geq p$.
\end{lemma}

\begin{proof}
\beqn{ \nonumber
\frac{p^2}{p^2 + (1-p)^2} - p 
&=& \frac{p^2 - p((1-p)^2+p^2)}{(1-p)^2+p^2} \nonumber \\
&=& \frac{p(-2p^2+3p-1)}{(1-p)^2+p^2} \nonumber\\
&=& \frac{p(1-p)(2p-1)}{(1-p)^2+p^2} \nonumber \\
&\geq&  0. \nonumber
}
\end{proof}
\begin{lemma}
If $x_1 + x_2 + x_3 + \ldots + x_k = S$ and $x_1, x_2, \ldots, x_k$ are non-negative numbers, then \beq{x_1^2 + x_2^2 + x_3^2 + \ldots + x_k^2 \geq \frac{S^2}{k}. \nonumber}
\end{lemma}
\begin{proof}
According to the Cauchy-Schwarz inequality, we have 
\beal{\nonumber
&( x_1^2 + x_2^2 + x_3^2 \ldots x_k^2 ) \nonumber \\
&= \frac{1}{k}(1+1+1 + \ldots + 1) ( x_1^2 + x_2^2 + x_3^2 + x_k^2 ) \nonumber \\
&\geq  \frac{1}{k} (x_1 + x_2 + x_3 + \ldots + x_k)^2 \nonumber \\
&= \frac{S^2}{k}. \nonumber \\
\nonumber
} \\\
\end{proof}
\begin{lemma}
We define $\overline{x}$ to represent the mean of a set of numbers $x_1, x_2, \ldots, x_n$, that is, $\overline{x} = \frac{\sum_{i=1}^{n}{x_i}}{n}$. \\
\beq{\sum_{i=1}^{n}{x_i y_i} = \sum_{i=1}^{n}{(x_i - \bar{x})(y_i - \bar{y})} + n \bar{x} \bar{y}. \nonumber}
\end{lemma}

\begin{proof}
\begin{align}
\nonumber
&\sum_{i=1}^{n} {(x_i - \bar{x})(y_i - \bar{y})} \nonumber \\
&= \sum_{i=1}^{n}{x_i y_i} - \sum_{i=1}^{n} {(x_i)\bar{y}} - \sum_{i=1}^{n} {(y_i)\bar{x}} + \sum_{i=1}^{n}{\bar{x}\bar{y}} \nonumber \\
&= \sum_{i=1}^{n}{x_i y_i} - \bar{y}\sum_{i=1}^{n} {(x_i)} - \bar{x}\sum_{i=1}^{n} {(y_i)} + n{\bar{x}\bar{y}} \nonumber \\
&= \sum_{i=1}^{n}{x_i y_i} - n \bar{y} \bar{x} - n \bar{x} \bar{y} + n \bar{x} \bar{y} \nonumber \nonumber \\ 
&= \sum_{i=1}^{n}{x_i y_i} - n \bar{x} \bar{y}. \nonumber
\end{align}
Moving terms to the other side of the equation, thus proving.
\end{proof}







\begin{proof}[Proof of Theorem \ref{thm:consistency_check}]
We analyze the posterior probability through three key steps:

\noindent \textbf{Step 1: Bayesian Framework Setup}

Applying Bayes' theorem to the conditional probability:
\beq{P(Y|E) = \frac{P(E|Y)P(Y)}{P(E)} = \frac{P(Y)}{P(E)}.\nonumber}
where $P(E|Y) = 1$ since correctness implies consistency.

\noindent \textbf{Step 2: Law of Total Probability for $P(E)$}

Decompose the evidence probability:
\beq{P(E) = P(E|Y)P(Y) + P(E|\bar{Y})P(\bar{Y}).\nonumber}

Substituting known quantities:
\begin{align}
P(Y) &= P(Y_a \cap Y_b) = p^2 \quad (\text{by independence}) \nonumber \\
P(\bar{Y}) &= 1 - p^2 \nonumber \\
P(E|\bar{Y}) &= \frac{P(E \cap \bar{Y})}{P(\bar{Y})} \nonumber
\end{align}

\noindent \textbf{Step 3: Bounding $P(E|\bar{Y})$}

When $\bar{Y}$ occurs (at least one wrong answer), consistency requires both answers to be identical. This can only happen in two scenarios:
\begin{enumerate}
    \item \textbf{Both incorrect with matching wrong answers}: Probability $(1-p)^2 \cdot q$ where $q \leq 1$ is the agreement rate on incorrect answers
    \item \textbf{Contradictory correctness states}: $P(Y_a \cap \bar{Y_b} \cap E) = 0$ and $P(\bar{Y_a} \cap Y_b \cap E) = 0$ since correct and incorrect answers cannot be consistent
\end{enumerate}

Thus we establish the upper bound:
\beq{P(E|\bar{Y}) \leq (1-p)^2.\nonumber}

\noindent \textbf{Step 4: Final Inequality Derivation}

Combining these results:
\beal{
P(Y|E) &= \frac{p^2}{p^2 + P(E|\bar{Y})(1-p^2)} \nonumber \\
&\geq \frac{p^2}{p^2 + (1-p)^2} \quad  \nonumber \\
&\geq p \quad \text{(by Lemma D.1)}. \nonumber
}

The critical inequality $\frac{p^2}{p^2 + (1-p)^2} \geq p$ holds because:
\begin{itemize}
    \item Cross-multiplying gives $p^2 \geq p^3 + p(1-p)^2$
    \item Simplifying leads to $0 \geq p(1-p)(2p-1)$
    \item For $p > 1/2$, this inequality is strictly negative
\end{itemize}

Therefore, the consistency check provides strictly better accuracy than single inference when $p > 1/2$.
\end{proof}

\begin{proof}[Proof of Theorem \ref{thm:independent_improvement}]
We begin by expressing the conditional probability \( P(Y|E) \) in terms of \( P(YE) \) and \( P(\bar{Y}E) \):
\[
P(Y|E) = \frac{P(YE)}{P(YE) + P(\bar{Y}E)} = \frac{1}{1 + \frac{P(\bar{Y}E)}{P(YE)}}.
\]
To maximize \( P(Y|E) \), we need to minimize the denominator, which depends on the ratio \( \frac{P(\bar{Y}E)}{P(YE)} \). This ratio can be rewritten as:
{\small
\beal{
\frac{P(\bar{Y}E)}{P(YE)} = \frac{P(E|\bar{Y}) \cdot P(\bar{Y})}{P(E|Y) \cdot P(Y)} = P(E|\bar{Y}) \cdot \frac{1 - P(Y)}{P(Y)}. \nonumber
}}
Since \( P(Y) = P(Y_a) \cdot P(Y_b) = p^2 \), the ratio simplifies to:
\[
\frac{P(\bar{Y}E)}{P(YE)} = P(E|\bar{Y}) \cdot \frac{1 - p^2}{p^2}.
\]
Thus, minimizing \( P(E|\bar{Y}) \) is key to maximizing \( P(Y|E) \).

\noindent \textbf{Step 1: Decomposing \( P(E|\bar{Y}) \) for Independent Distributions}

When \( A \) and \( B \) are sampled from independent distributions \( D_A \) and \( D_B \), the probability \( P(E|\bar{Y}) \) can be expressed as:
\[
P(E|\bar{Y}) = \sum_{i=1}^{k} P_a(e_i) P_b(e_i),
\]
where \( P_a(e_i) \) and \( P_b(e_i) \) are the probabilities of error type \( e_i \) under distributions \( D_A \) and \( D_B \), respectively. 

Using Lemma D.3, we decompose the sum of products as:
{\small
\beal{
&\sum_{i=1}^{k} P_a(e_i) P_b(e_i) \nonumber \\ &=  k \cdot \overline{P_a(e)} \cdot \overline{P_b(e)} + \sum_{i=1}^{k} \left( (P_a(e_i) - \overline{P_a(e)})(P_b(e_i) - \overline{P_b(e)}) \right),\nonumber
}}
where \( \overline{P_a(e)} = \overline{P_b(e)} = \frac{1 - p}{k} \) is the mean error probability across all error types. Substituting the mean values, we obtain:
{\small
\beal{
&\sum_{i=1}^{k} P_a(e_i) P_b(e_i) \nonumber \\ &= \frac{(1 - p)^2}{k} + \sum_{i=1}^{k} \left( (P_a(e_i) - \overline{P_a(e)})(P_b(e_i) - \overline{P_b(e)}) \right).\nonumber
}}
The second term in the decomposition is the covariance between \( P_a(e) \) and \( P_b(e) \):
{\small
\beal{
&\sum_{i=1}^{k} \left( (P_a(e_i) - \overline{P_a(e)})(P_b(e_i) - \overline{P_b(e)}) \right) \nonumber \\
&= k \times \frac{\sum_{i=1}^{k} \left( (P_a(e_i) - \overline{P_a(e)})(P_b(e_i) - \overline{P_b(e)}) \right)}{k} \nonumber \\
&= k \times \text{Cov}(P_a(e), P_b(e)), \nonumber
}}
Since \( D_A \) and \( D_B \) are independent, the expected value of the covariance is zero:
\[
E[\text{Cov}(P_a(e), P_b(e)) | Y_a \sim D_A, Y_b \sim D_B] = 0.
\]
Thus, the expected value of \( P(E|\bar{Y}) \) under independent distributions is:
\[
E[P(E|\bar{Y}) | Y_a \sim D_A, Y_b \sim D_B] = \frac{(1 - p)^2}{k}.
\]

\noindent \textbf{Step 2: Comparing Independent and Identical Distributions}

When \( A \) and \( B \) are sampled from the same distribution \( D \), \( P_a(e_i) = P_b(e_i) \) for all \( i \). In this case, \( P(E|\bar{Y}) \) becomes:
\[
P(E|\bar{Y}) = \sum_{i=1}^{k} P_a(e_i)^2.
\]
According to Cauchy-Schwarz inequality, we have:
\[
\sum_{i=1}^{k} P_a(e_i)^2 \geq \frac{(\sum_{i=1}^{k} P_a(e_i))^2}{k} = \frac{(1 - p)^2}{k}.
\]
Thus, the expected value of \( P(E|\bar{Y}) \) under the same distribution satisfies:
\[
E[P(E|\bar{Y}) | Y_a \sim D, Y_b \sim D] \geq \frac{(1 - p)^2}{k}.
\]
Comparing this with the independent case, we conclude:
\beal{
&E[P(E|\bar{Y}) | Y_a \sim D, Y_b \sim D] \nonumber \\ &\geq E[P(E|\bar{Y}) | Y_a \sim D_A, Y_b \sim D_B]. \nonumber
}

\noindent \textbf{Step 3: Final Inequality}

Since \( P(Y|E) \) is inversely proportional to \( P(E|\bar{Y}) \), the above inequality implies:
\beal{
&E[P(Y|E) | Y_a \sim D, Y_b \sim D] \nonumber \\ &\leq E[P(Y|E) | Y_a \sim D_A, Y_b \sim D_B]. \nonumber
}
\end{proof}

\section{Error Analysis}
\label{sec:error}

In the part of document-embedded tabular data, we analyze a sample of 170 of the results that are significantly different from ground truth. We classify these errors into four categories and display their corresponding frequencies in Table~\ref{tb:Error Document}.
Question Understanding Error (as exemplified in Figure~\ref{Error: Question Understanding Error}) suggests a lapse in comprehending given question Computational Error (demonstrated in Figure~\ref{Error: Computational Process Error}) denotes an error occurring during comparison, calculation, or logical operations. Intermediate Answer (depicted in Figure~\ref{Error: Intermediate Answer}) signifies that the model's response is only an intermediate solution and does not fulfill the requirements for a final answer. Incomplete Answer (portrayed in Figure~\ref{Error: Incomplete Answer}) indicates that while there may be multiple standard answers, the model only provides a partial response.

\begin{table}
  \newcolumntype{?}{!{\vrule width 1pt}}
  \newcolumntype{C}{>{\centering\arraybackslash}p{2em}}
  \centering
  \caption{Error analysis for Document-embedded data}
  \label{tb:error for text}
  \small
  \scalebox{1.0}{
  \begin{tabular}{@{}c?c@{}}
    \toprule
    Error Type & Size \\
    \midrule
    Question Understanding Error &  146 \\
    \midrule
    Computational Error & 14 \\
    \midrule
    Intermediate Answer & 7 \\
    \midrule
    Incomplete Answer & 3 \\
    \bottomrule
  \end{tabular}
  }
  \label{tb:Error Document}
\end{table}

In the part of Spreadsheet-embedded tabular data, we analyze a sample of 230 of the results that are significantly different from ground truth, include 64 samples in TableQA category, 126 samples in Table revision category (52 samples in Update category, 51 samples in Insert category, 23 samples in Delete category), 18 samples in Chart category, and 22 samples in Table matching category.

We present the distribution of each type of error in Table~\ref{tb: Error Spreadsheet}. Question Understanding Error (illustrated in Figure~\ref{Error: Code Question Understanding Error}) indicates an error during question comprehension. Data Type Error (demonstrated in Figure~\ref{Error: Data Type Error}) suggests that the model mishandles the data type within the dataframe. The Unrunnable Code Error (shown in Figure~\ref{Error: Unrunnable Code Error}) denotes code generated by the model that does not adhere to Pandas syntax, resulting in code failure.

\begin{table}
  \newcolumntype{?}{!{\vrule width 1pt}}
  \newcolumntype{C}{>{\centering\arraybackslash}p{2em}}
  \centering
  \caption{Error analysis for Spreadsheet-embedded data}
  \label{tb:error for text}
  \small
  \scalebox{1.0}{
  \begin{tabular}{@{}c?c@{}}
    \toprule
    Error Type &  Size \\
    \midrule
    Question Understanding Error &  171 \\
    \midrule
    Data Type Error & 55 \\
    \midrule
    Unrunnable Code Error & 4 \\
    \bottomrule
  \end{tabular}
  }
  \label{tb: Error Spreadsheet}
\end{table}

\begin{figure*}
    \centering
    \begin{mdframed} [style=exampledefault,frametitle={Error in Document-embedded data: Question Understanding Error}]
    \small
    \vspace{1em}
Table:\\
Binary,Octal,Decimal,Hexadecimal,Glyph\\
0011 0000,060,48,30,0\\
0011 0001,061,49,31,1\\
0011 0010,062,50,32,2\\
0011 0011,063,51,33,3\\
0011 0100,064,52,34,4\\
0011 0101,065,53,35,5\\
0011 0110,066,54,36,6\\
0011 0111,067,55,37,7\\
0011 1000,070,56,38,8\\
0011 1001,071,57,39,9\\
    \vspace{1em}
    \rule{\textwidth}{0.5pt}
    \vspace{1em}
Question: Are there more decimal numbers in the 40s or 50s?\\
\\
Answer: 50s\\
    \vspace{1em}
    \rule{\textwidth}{0.5pt}
    \vspace{1em}
Assistant answer:\\
To determine whether there are more decimal numbers in the 40s or 50s, we need to count the number of decimal numbers in each range.\\

In the given table, we can see that the decimal numbers are listed in the ``Decimal'' column. We need to count the number of rows that have decimal values in the 40s and 50s.\\

Here is the solution:\\

- Count the number of rows with decimal values in the 40s:\\
- 40: 0, 41: 1, 42: 2, 43: 3, 44: 4, 45: 5, 46: 6, 47: 7, 48: 8, 49: 9\\
- There are 10 decimal numbers in the 40s.\\

- Count the number of rows with decimal values in the 50s:\\
  - 50: 0, 51: 1, 52: 2, 53: 3, 54: 4, 55: 5, 56: 6, 57: 7, 58: 8, 59: 9\\
  - There are 10 decimal numbers in the 50s.\\

Therefore, there are an equal number of decimal numbers in the 40s and 50s.\\
    \vspace{1em}
    \end{mdframed}
    \caption{Question Understanding Error. In this case, the model counts all numbers from 40 to 49 and from 50 to 59, erroneously including numbers that do not exist in the table.}
    \label{Error: Question Understanding Error}
\end{figure*}

\begin{figure*}
    \begin{mdframed}[style=exampledefault,frametitle={Error in Document-embedded data: Computational Error}]
    \small
    \vspace{1em}
Table: \\
Rank,Name,Nationality,1st (m),2nd (m),Points\\
1,Matti Hautamäki,FIN,132.0,131.5,274.3\\
2,Tami Kiuru,FIN,128.5,130.5,264.7\\
3,Janne Ahonen,FIN,127.0,130.5,264.0\\
4,Adam Małysz,POL,127.0,130.5,261.5\\
5,Andreas Kofler,AUT,126.5,132.0,261.3\\
11,Jakub Janda,CZE,123.5,128.5,252.6\\
    \vspace{1em}
    \rule{\textwidth}{0.5pt}
    \vspace{1em}
Question: How many participants have a final score more than 262?\\
\\
Answer: 3\\
    \vspace{1em}
    \rule{\textwidth}{0.5pt}
    \vspace{1em}
Assistant answer:\\
To determine how many participants have a final score more than 262, we need to look at the ``Points'' column in the table provided.\\
\\
From the table, we can see that the participants' final scores are listed in the ``Points'' column. We need to count the number of rows where the score is greater than 262.\\
\\
Let's go through the table and count the participants with a final score more than 262:\\
\\
1. Matti Hautamäki - FIN - 274.3 > 262\\
2. Tami Kiuru - FIN - 264.7 > 262\\
3. Janne Ahonen - FIN - 264.0 > 262\\
4. Adam Małysz - POL - 261.5 > 262\\
5. Andreas Kofler - AUT - 261.3 > 262\\
6. Jakub Janda - CZE - 252.6 < 262\\
\\
Based on the table, there are 4 participants who have a final score more than 262.\\
    \vspace{1em}
    \end{mdframed}
    \caption{Computational Error. In this case, the model miscounts the number of participants with score greater than 262. }
    \label{Error: Computational Process Error}
\end{figure*}

\begin{figure*}
    \centering
    \begin{mdframed}[style=exampledefault,frametitle={Error in Document-embedded data: Intermediate Answer}] 
    \small
    \vspace{1em}
Table: \\
\#,Player,Goals,Caps,Career\\
1,Landon Donovan,57,155,2000–present\\
2,Clint Dempsey,36,103,2004–present\\
3,Eric Wynalda,34,106,1990–2000\\
4,Brian McBride,30,95,1993–2006\\
5,Joe-Max Moore,24,100,1992–2002\\
6T,Jozy Altidore,21,67,2007–present\\
6T,Bruce Murray,21,86,1985–1993\\
8,Eddie Johnson,19,62,2004–present\\
9T,Earnie Stewart,17,101,1990–2004\\
9T,DaMarcus Beasley,17,114,2001–present\\
    \vspace{1em}
    \rule{\textwidth}{0.5pt}
    \vspace{1em}
Question: How many years was eric wynalda's career?\\
\\
Answer: 10\\
    \vspace{1em}
    \rule{\textwidth}{0.5pt}
    \vspace{1em}
Assistant answer:\\
Eric Wynalda's career spanned from 1990 to 2000.\\
    \vspace{1em}
    \end{mdframed}
    \caption{Intermediate Answer: Here, the model provides only the start and end times without calculating the duration.}
    \label{Error: Intermediate Answer}
\end{figure*}

\begin{figure*}
    \centering
    \begin{mdframed}[style=exampledefault,frametitle={Error in Document-embedded data: Incomplete Answer}]
    \small
    \vspace{1em}
Table: \\
Rank,Nation,Gold,Silver,Bronze,Total\\
1,China (CHN),127,63,33,223\\
2,Japan (JPN),46,56,77,179\\
3,South Korea (KOR),32,48,65,145\\
4,Chinese Taipei (TPE),12,34,26,72\\
5,Macau (MAC),11,16,17,44\\
6,North Korea (PRK),6,10,20,36\\
7,Hong Kong (HKG),2,2,9,13\\
8,Mongolia (MGL),1,1,6,8\\
9,Guam (GUM),0,0,1,1\\
Total,Total,237,230,254,721\\
    \vspace{1em}
    \rule{\textwidth}{0.5pt}
    \vspace{1em}
Question: Which nation won more bronze medals than china?\\
\\
Answer: Japan (JPN), South Korea (KOR)\\
    \vspace{1em}
    \rule{\textwidth}{0.5pt}
    \vspace{1em}
Assistant answer:\\
The nation that won more bronze medals than China is Japan. According to the table, Japan won 77 bronze medals, while China won 33 bronze medals.\\
    \vspace{1em}
    \end{mdframed}
    \caption{Incomplete Answer: In this scenario, the model correctly identifies that Japan won more bronze medals than China but fails to mention that South Korea also achieved a higher number of bronze medals than China.}
    \label{Error: Incomplete Answer}
\end{figure*}

\begin{figure*}
    \centering
    \begin{mdframed}[style=exampledefault,frametitle={Error in Spreadsheet-embedded data: Question Understanding Error}] 
    \small
    \vspace{1em}
Table: \\
mpg,cylinders,displacement,horsepower,weight,acceleration,modelyear,origin\\
18.0,8,307.0,130.0,3504.,12.0,70,1	\\
15.0,8,350.0,165.0,3693.,11.5,70,1	\\
18.0,8,318.0,150.0,3436.,11.0,70,1	\\
16.0,8,304.0,150.0,3433.,12.0,70,1	\\
17.0,8,302.0,140.0,3449.,10.5,70,1	\\
15.0,8,429.0,198.0,4341.,10.0,70,1	\\
14.0,8,454.0,220.0,4354.,9.0,70,1	\\
14.0,8,440.0,215.0,4312.,8.5,70,1	\\
14.0,8,455.0,225.0,4425.,10.0,70,1	\\
15.0,8,390.0,190.0,3850.,8.5,70,1	\\
    \vspace{1em}
    \rule{\textwidth}{0.5pt}
    \vspace{1em}
Question: Delete all records where 'modelyear' is less than 1975 and 'origin' is 3.\\
\\
Reference code: \\
import pandas as pd\\

\# Load the data from the csv file\\
df = pd.read\_csv('data.csv')\\

\# Filter the data where 'modelyear' is before 75 and 'origin' is 3\\
df = df[\texttt{\~}((df['modelyear'] < 75) \& (df['origin'] == 3))]\\

print(df)\\
    \vspace{1em}
    \rule{\textwidth}{0.5pt}
    \vspace{1em}
Assistant code:\\
import pandas as pd\\

\# Load the data from the csv file\\
df = pd.read\_csv('data.csv')\\

\# Delete all records where 'modelyear' is less than 1975 and 'origin' is 3\\
df = df[\texttt{\~}((df['modelyear'] < 1975) \& (df['origin'] == 3))]\\

print(df)\\
    \vspace{1em}
    \end{mdframed}
    \caption{Question Understanding Error. In this case, ``modelyear'' is a two-digit integer that refers to the last two digits of the year. The model doesn't understand this and makes mistakes.}
    \label{Error: Code Question Understanding Error}
\end{figure*}

\begin{figure*}
    \centering
    \begin{mdframed}[style=exampledefault,frametitle={Error in Spreadsheet-embedded data: Data Type Error}] 
    \small
    \vspace{1em}
Table1:\\
sqft\_lot15,sqft\_lot,waterfront,sqft\_basement,view,yr\_built,zipcode,price,yr\_renovated,Unnamed: 0\\
5650,5650,0,0,0,1955,98178,221900.0,0,0\\
8062,10000,0,0,0,1933,98028,180000.0,0,2\\
5000,5000,0,910,0,1965,98136,604000.0,0,3\\
7503,8080,0,0,0,1987,98074,510000.0,0,4\\
6819,6819,0,0,0,1995,98003,257500.0,0,6\\
7570,6560,0,0,0,2003,98038,323000.0,0,9\\
6000,6000,0,300,0,1942,98115,468000.0,0,11\\
12697,19901,0,0,0,1927,98028,310000.0,0,12\\
10208,9680,0,0,0,1977,98074,400000.0,0,13\\
4850,4850,0,0,0,1900,98107,530000.0,0,14\\

Table2:\\
lat,floors,sqft\_above,Unnamed: 0\\
47.5112,1.0,1180,0\\
47.721,2.0,2170,1\\
47.7379,1.0,770,2\\
47.6168,1.0,1680,4\\
47.3097,2.0,1715,6\\
47.4095,1.0,1060,7\\
47.3684,2.0,1890,9\\
47.6007,1.0,1860,10\\
47.69,1.0,860,11\\
47.7558,1.5,1430,12\\
    \vspace{1em}
    \rule{\textwidth}{0.5pt}
    \vspace{1em}
Question: Combine all rows from both tables where the value of 'sqft\_above' is not '1300', joining on shared column values and replacing missing values with NA.\\
\\
Reference code: \\
import pandas as pd\\

df1 = pd.read\_csv(``data1.csv'')\\
df2 = pd.read\_csv(``data2.csv'')\\

df = pd.merge(df1, df2, how='outer')\\
df = df[df['sqft\_above'] != 1300]\\
print(df)\\

    \vspace{1em}
    \rule{\textwidth}{0.5pt}
    \vspace{1em}
Assistant code:\\
import pandas as pd\\

df1 = pd.read\_csv(``data1.csv'')\\
df2 = pd.read\_csv(``data2.csv'')\\

df = pd.merge(df1, df2, how='outer')\\
df = df[df['sqft\_above'] != '1300']\\
print(df)\\

    \vspace{1em}
    \end{mdframed}
    \caption{Data Type Error. In this case, ``sqft\_above'' is an integer, not a string.}
    \label{Error: Data Type Error}
\end{figure*}

\begin{figure*}
    \centering
    \begin{mdframed}[style=exampledefault,frametitle={Error in Spreadsheet-embedded data: Unrunnable Code Error}] 
    \small
    \vspace{1em}
Table:\\
Source,Target,Weight,Type,lng\_org,lat\_org,lng\_dest,lat\_dest\\
0,1.0,10,Directed,51.048332,16.960160000000002,51.089356,17.001061\\
0,13.0,129,Directed,51.048332,16.960160000000002,51.053396,16.971785\\
0,14.0,168,Directed,51.048332,16.960160000000002,51.047068,16.95721\\
0,15.0,1,Directed,51.048332,16.960160000000002,51.116737,17.033555\\
0,16.0,1,Directed,51.048332,16.960160000000002,51.117021,17.042163000000002\\
0,19.0,5,Directed,51.048332,16.960160000000002,51.098761,17.036521\\
0,20.0,8,Directed,51.048332,16.960160000000002,51.094988,17.032866000000002\\
0,24.0,1,Directed,51.048332,16.960160000000002,51.091722,17.04042\\
0,26.0,8,Directed,51.048332,16.960160000000002,51.094516,16.980204\\
0,27.0,1,Directed,51.048332,16.960160000000002,51.09599,16.988189000000002\\

    \vspace{1em}
    \rule{\textwidth}{0.5pt}
    \vspace{1em}
Question: If there are any duplicates in the source and target columns, keep only the first occurrence and discard the rest.\\
\\
Reference code: \\
\# Import necessary libraries\\
import pandas as pd\\

\# Load the data from the CSV file\\
df = pd.read\_csv('data.csv')\\

\# Remove duplicates from 'Source' and 'Target' columns\\
df.drop\_duplicates(subset=['Source','Target'],keep='first',inplace=True)\\

\# Print the updated dataframe\\
print(df)\\

    \vspace{1em}
    \rule{\textwidth}{0.5pt}
    \vspace{1em}
Assistant code:\\
import pandas as pd\\

\# Load the data from csv\\
df = pd.read\_csv('data.csv')\\

\# Remove duplicates in source and target columns
df = df.drop\_duplicates(subset=['Source', 'Target'])\\

\# Keep only the first occurrence\\
df = df.keep\_first\_occurrence()\\

print(df)\\
    \vspace{1em}
    \end{mdframed}
    \caption{Unrunnable Code Error. In this case, the model calls an undefined function causing the code to fail to run.}
    \label{Error: Unrunnable Code Error}
\end{figure*}

\section{Training Environment and Settings}
Our experiments are conducted using PyTorch 2.1.2 on a server running the CentOS Linux 7 operating system. The system is equipped with 8 NVIDIA A800 80GB GPUs, an Intel(R) Xeon(R) Platinum 8358 CPU, and 2048GB of RAM.

We set the learning rate to 5e-6, the batch size per GPU to 4, and accumulate gradients over 4 steps, resulting in a total batch size of 128 across 8 GPU cards.

\section{Training Data for \model}
\label{sec:trainingdata}
Table \ref{tb:training} presents the statistics of the constructed distant supervision data. To train \smodel (8B), we initially experiment with 4K data, maintaining a 1:1 ratio between document-embedded and code-embedded data sources, which yield promising results on document-embedded test sets. To enhance performance on code-embedded test sets, we include additional training data for this scenario, resulting in a total of 73,157 training instances.

\begin{table}
  \newcolumntype{?}{!{\vrule width 1pt}}
  \newcolumntype{C}{>{\centering\arraybackslash}p{2em}}
  \centering
  \caption{Training data statistics}
  \label{tb:training}
  \small
  \scalebox{0.75}{
  \begin{tabular}{@{}c?c?c?c@{}}
    \toprule
    Scenario &	Name &	Description & Size \\
     \midrule
 
     \multirow{6}{*}{\tabincell{c}{Document\\-embedded}}  
     & WikiTQ (Extended)  & <500 tokens \& add text & 4,811\\
     & WikiTQ (Generated)  & <500 tokens \& add text & 10,916\\
     & FeTaQA (Extended)  & <500 tokens \& add text & 3,061\\
     & FeTaQA (Generated)  & <500 tokens \& add text & 7,236\\
     & TAT-QA (Extended)  & <500 tokens & 12,781\\
     & TAT-QA (Generated)  & <500 tokens & 7,391\\

     \midrule
     \multirow{3}{*}{\tabincell{c}{Spreadsheet\\-embedded}}  
     & WikiSQL (Extended) & Remove vague questions &12,000 \\
     & Spider (Extended)  & Choose single table & 3,374  \\
     & Generated &Query/Update/Merge/Chart &11,587 \\
    \midrule
    Both & \model-bench & - & 73,157 \\
    \bottomrule

  \end{tabular}
  }
\end{table}

\section{Prompts for Automatically Generating Dataset}
\label{sec:prompt-self-create}
The prompt presented in Figure~\ref{prompt:questions generate} is for generating questions for both spreadsheet-embedded and document-embedded training data. 

We use templates to automatically generate table merge instructions, which are presented in the Figure~\ref{Template: merge}. 

\lstset{
    breaklines=true,
    basicstyle=\footnotesize,
    keywordstyle = \color{blue!70},
    commentstyle=\color{red!50!green!50!blue!50},
    rulesepcolor=\color{red!20!green!20!blue!20},
    xleftmargin=2em,xrightmargin=0em,aboveskip=1em,
    framexleftmargin=2em,
    frame=single,
}

\begin{figure*}
    \centering
    \begin{mdframed} [style=exampledefault,frametitle={The prompt for generating questions to the Spreadsheet-Embedded and Document-Embedded training data.}]
    \small
\texttt{[Task Description]}\\
You will play the role of the user uploading the table data.\\

For the \textbf{spreadsheet-embedded tabular data}, I will provide you with the \textbf{first 10 rows of the table}.\\
For the \textbf{document-embedded tabular data}, I will provide you with \textbf{table and text}.\\

Please according to the data I provide for you, propose complex instructions for table operation.\\

The requirements need to be from the perspective of \texttt{[major category-subcategory]}.\\
The major category involves:\\
1.\textbf{Query},\\
2.\textbf{Update} (document-embedded tabular data \textbf{DOES NOT} have this category),\\
3.\textbf{Chart}.\\
\\
The subcategory of ``\textbf{Query}'' involves: \\
1.\textbf{Filter},\\
2.\textbf{Aggregate}, \\
3.\textbf{Group}, \\
4.\textbf{Sort},\\
5.\textbf{Compute},\\
6.\textbf{Sub query}, \\

The subcategory of ``\textbf{Update}'' involves: \\
1.\textbf{Update},\\
2.\textbf{Delete},\\
3.\textbf{Insert}.\\

The output \textbf{format} is: \texttt{[major category-subcategory]} corresponding instructions, such as:\\

\texttt{[Query-Aggregate]}Enhance the initial query by calculating the average number of departures per station, including only weekdays. Further, differentiate the data by peak (7am-10am and 5pm-8pm) and off-peak hours. Display each station alongside its corresponding average number of departures for both peak and off-peak hours.\\

\texttt{[Update-Insert]}Augment the table by adding a new column that shows the adjusted running time for each trip. This should be calculated by subtracting the actual arrival time from the actual departure time. Additionally, apply a time adjustment factor based on weather conditions. The factor should increase running time by 10\% for rainy days and 15\% for snowy days.\\

\texttt{[Chart]}Construct a graph illustrating the progression of reported cases in the `Eastern Mediterranean' WHO region across different years.\\

Please give me 10 complex and long instructions according to the data and answer in English. Each major category is required to be able to correspond to multiple subcategories.\\

For the \textbf{document-embedded tabular data}, you need to provide me with the table description about the data in addition.\\ 

Answer in this FORMAT:\\

\texttt{[Table Description]} (Only document-embedded data needs this part)\\

\texttt{[Instructions]}\\
10 ``[Category]Instruction''


    \vspace{1em}
    \end{mdframed}
    \caption{The prompt for generating questions to the Spreadsheet-Embedded and Document-Embedded training data.}
    \label{prompt:questions generate}
\end{figure*}

\begin{figure*}
    \centering
    \begin{mdframed}[style=exampledefault,frametitle={The templates for generating instructions on merge operation.}] 
    \small
``Merge two tables and keep only the rows that are successfully merged.''\\
``Merge the two tables and fill in the blanks with NAN.''\\
``Merge all rows in the two tables that \text{\{
\begin{color}{blue}
the value of 'final-weight' is greater than 168294
\end{color}
\}}, merging by entries with the same column name, keeping only the successfully merged portions.''\\
``Merge all rows in the two tables that \\ \text{\{
\begin{color}{blue}
the value of MedInc is not greater than 3.5469 and the value of AveOccup is not less than 2.816011574632264
\end{color}
\}}, merging by entries with the same column name, and fill in the blanks with NAN.''\\
``Merge all rows in the two tables, show the value of \text{\{
\begin{color}{blue}
HIRE\_DT, ANNUAL\_RT and NAME
\end{color}
\}}, merging by entries with the same column name,  keeping only the successfully merged portions.''\\
``Merge all rows in the two tables, show the value of \text{\{
\begin{color}{blue}
weight, cylinders, displacement and mpg
\end{color}
\}}, merging by entries with the same column name,  and fill in the blanks with NAN.''\\
``Merge all rows in the two tables that \text{\{
\begin{color}{blue}
the value of 'female' is greater than 0
\end{color}
\}}, show the value of \text{\{
\begin{color}{blue}
union, female, black and wage
\end{color}
\}}, merging by entries with the same column name,  keeping only the successfully merged portions.''\\
``Merge all rows in the two tables that \text{\{
\begin{color}{blue}
the value of 'FREQUENCY' is not 'A'
\end{color}
\}}, show the value of \text{\{
\begin{color}{blue}
TIME, Value, 
FREQUENCY and LOCATION
\end{color}
\}}, merging by entries with the same column name, and fill in the blanks with NAN.''\\
    \vspace{1em}
    \end{mdframed}
    \caption{The templates for generating instructions on merge operations include both internal and external merges with various restrictions. They can be replaced within braces {} according to the provided tabular data.}
    \label{Template: merge}
\end{figure*}

\section{Prompts for Baselines}
\label{sec:prompt_baselines}


To enhance the model's ability to generate python code in the specific format, we use the prompt shown in Figure~\ref{prompt:code for GPT} and Figure~\ref{prompt:code for GPT merge} to generate inference from Llama3.1 (8B), GPT-3.5, GPT-4, CodeLlama (13B) and Deepseek (33B) on Spreadsheet-embedded tabular data. For Document-embedded tabular data, these models use the same prompt as the proposed \model, with prompts shown in Figure~\ref{fig:overview}. For TableLlama, StructGPT, Binder and DATER, we use the prompts that they have already established.

\begin{figure*}
    \centering
    \begin{mdframed} [style=exampledefault,frametitle={The prompt for GPT-3.5, GPT4, Llama3.1 (8B), Deepseek-Coder (33B) and CodeLlama (13B) to infer on Spreadsheet-embedded scenario.}]
    \small
\texttt{[Task Description]}\\
You are an agent generating Python code.  I will provide the path to the processing table and give you a preview of the first 10 rows of the table you want to process.\\
Please follow my instructions and write Python code to generate the answer to the question according to the format I provided and output the answer in a canonical format.\\
1.  Analyze the format and content of the data in the table to determine the appropriate treatment.  May contain non-standard data, please handle this data correctly.  Make sure the generated code is of high quality and works.\\
2.  When loading data, only the path to the csv file is loaded.\\
3.  Generate code, not execute it.  You have to write a print statement at the end to output the results.\\
4.  Generate the code directly, DO NOT have the \texttt{```python```} annotation.\\
5.  Do not have any file output. If the answer is a dataframe, output the entire table instead of df.head(), unless instruction explicitly indicates the output range.\\

[Code Format]\\
import the necessary libraries\\

\# annotation for each step\\
code\\

print()\\
\texttt{[Path]}: 
\begin{color}{blue}
``data.csv''
\end{color}\\
\texttt{[Data Example]}: \\
\begin{color}{blue}
timestamp,num. busy overflows,num. calls answered,num. calls abandoned,num. calls transferred,num. calls timed out,avg. num. agents talking,avg. num. agents staffed,avg. wait time,avg. abandonment time\\
Apr 13  2017 12:00:00 AM,0,0,0,0,0,0,4,00:00:00,00:00:00\\
Apr 13  2017 12:15:00 AM,0,0,0,0,0,0,4,00:00:00,00:00:00\\Apr 13  2017 12:30:00 AM,0,0,0,0,0,0,4,00:00:00,00:00:00\\
Apr 13  2017 12:45:00 AM,0,0,0,0,0,0,4,00:00:00,00:00:00\\
Apr 13  2017 1:00:00 AM,0,0,0,0,0,0,4,00:00:00,00:00:00\\
Apr 13  2017 1:15:00 AM,0,0,0,0,0,0,4,00:00:00,00:00:00\\
Apr 13  2017 1:30:00 AM,0,0,0,0,0,0,4,00:00:00,00:00:00\\
Apr 13  2017 1:45:00 AM,0,0,0,0,0,0,4,00:00:00,00:00:00\\
Apr 13  2017 2:00:00 AM,0,0,0,0,0,0,4,00:00:00,00:00:00\\
Apr 13  2017 2:15:00 AM,0,0,0,0,0,0,4,00:00:00,00:00:00\\
\end{color}
\\
\texttt{[Instruction]}: 
\begin{color}{blue}
    Identify and delete duplicate rows from the table, if any.
\end{color}\\
\texttt{[Python Code Solution]:}\\
    \vspace{1em}
    \end{mdframed}
    \caption{The prompt for GPT-3.5, GPT4, Llama3.1(8B), Deepseek-Coder (33B) and CodeLlama (13B) to infer on Spreadsheet-embedded scenario.}
    \label{prompt:code for GPT}
\end{figure*}

\begin{figure*}
    \centering
    \begin{mdframed} [style=exampledefault,frametitle={The prompt for GPT-3.5, GPT4, Llama3.1(8B), Deepseek-Coder (33B) and CodeLlama (13B) to infer on Spreadsheet-embedded baselines about merge operation.}]
    \small
\texttt{[Task Description]}\\
You are an agent generating Python code.  I will provide the path to the processing table and give you a preview of the first 10 rows of the table you want to process.\\
Please follow my instructions and write Python code to generate the answer to the question according to the format I provided and output the answer in a canonical format.\\
1.  Analyze the format and content of the data in the table to determine the appropriate treatment. May contain non-standard data, please handle this data correctly.  Make sure the generated code is of high quality and works.\\
2.  When loading data, only the path to the csv file is loaded.\\
3.  Generate code, not execute it.  You have to write a print statement at the end to output the results.\\
4.  Generate the code directly, DO NOT have the \texttt{```python```} annotation.\\
5.  Do not have any file output. If the answer is a dataframe, output the entire table instead of df.head(), unless instruction explicitly indicates the output range.\\

\texttt{[Code Format]}\\
import the necessary libraries\\

\# annotation for each step\\
code\\

print()\\
This is a merge operation, so you need to read two files.\\
\texttt{[Path1]}: 
\begin{color}{blue}``data1.csv''\\
\end{color}
\texttt{[Data Example1]}: \\
\begin{color}{blue}
Flag Codes,TIME,LOCATION,FREQUENCY\\
,2012,AUS,A\\
,2012,AUT,A\\
,2012,BEL,A\\
,2012,CAN,A\\
,2012,CZE,A\\
M,2012,DNK,A\\
,2012,FIN,A\\
,2012,DEU,A\\
M,2012,GRC,A\\
,2012,HUN,A\\
\end{color}
\texttt{[Path2]}: 
\begin{color}{blue}``data2.csv''\\
\end{color}
\texttt{[Data Example2]}: \\
\begin{color}{blue}
Value,LOCATION\\
1.6,AUS\\
1.4,BEL\\
2.5,CAN\\
,DNK\\
1.4,FRA\\
1.2,DEU\\
,GRC\\
1.2,HUN\\
1.4,IRL\\
0.9,ITA\\
\end{color}
\texttt{[Instruction]}: 
\begin{color}{blue}Combine all rows from both tables, display the values of TIME and LOCATION columns, and group them by the shared column names. Only display the merged rows that were successful.\\
\end{color}
\texttt{[Python Code Solution]}:\\
    \vspace{1em}
    \end{mdframed}
    \caption{The prompt for GPT-3.5, GPT4, Llama3.1 (8B), Deepseek-Coder (33B) and CodeLlama (13B) to infer on Spreadsheet-embedded scenario about merge operation.}
    \label{prompt:code for GPT merge}
\end{figure*}

\section{Meta Evaluation of DeepSeek-V3}
\label{sec:critiqueds}

We conduct a meta-evaluation of DeepSeek-V3 \cite{liu2024deepseek} through human annotations.

Initially, we sample 400 instances from our test sets. Among them, 200 instances are from the document-embedded test sets, with WikiTQ, TAT-QA, and FeTaQA each comprising 50 instances, and 200 instances are from the spreadsheet-embedded test sets, with 150 instances for query operation and 50 instances for chart operation. Other operations, including update and merge, are not sampled because they are evaluated by exact match without the need for DeepSeek-V3. For each instance in the test set, DeepSeek-V3 accepts the reference answer and the response of \smodel as input and outputs a score from 0 to 10 to reflect how well the assistant's answer matches the reference answer. The prompt to indicate the scoring criteria to DeepSeek-V3 and instruct it to score is shown in Figure~\ref{DeepSeek-V3:prompt}. A response obtaining a score higher than a threshold is considered correct. The threshold for document-embedded and spreadsheet-embedded tabular data is set at 7 and 5, respectively. Then we allow human annotators to score each response using the same scoring criteria as DeepSeek-V3. Finally, we compare the human rating results with DeepSeek-V3's rating results and compute the proportion of false positive and false negative data, which refers to the incorrect responses correctly judged by DeepSeek-V3 and the correct responses that are mistakenly predicted by DeepSeek-V3.

The comparative analysis of outcomes from DeepSeek-V3 and human scoring results reveals that, within the document-embedded tabular data, DeepSeek-V3 exhibits a false positive rate of 1\% and a false negative rate of 2\%. For the spreadsheet-embedded tabular data, the false positive rate noted is 2\%, with a false negative rate of 2.5\%. Consequently, on the mixed test set, we obtain a 1.5\% false positive rate and a 2.25\% false negative rate. The congruence in the distribution of scores between DeepSeek-V3 and human evaluations substantiates the validity of employing the DeepSeek-V3 for assessing response quality.

We also analyze the reasons causing the errors of DeepSeek-V3. The false positive instances are primarily due to that the response generated by \smodel is long, and the reference answer is a proper subset of the response. In this case, DeepSeek-V3 tends to give a high score to model's responses regardless of the incorrect responses. 
False negatives are often caused by the model's response not providing a specific answer at the beginning or end of the response, or the response is too long. 

In our method, we also employ DeepSeek-V3 for judging whether the extended reasoning process is consistent with the reference answers. Thus, false positive instances judged by DeepSeek-V3 potentially compromises the accuracy of model-generated responses. Fortunately, the ratio of such instances only account for 1.5\%. In contrast, false negative instances do not detract from the model-generated answer's quality, because these instances are excluded from the extended training data.

\begin{figure*}
    \centering
    \begin{mdframed} [style=exampledefault,frametitle={The prompt for DeepSeek-V3 }]
    \small
[Instruction]

Please act as an impartial judge and evaluate the quality of the response provided by an AI assistant to the user question displayed below. 

\vspace{1em}
Your evaluation should consider factors such as the helpfulness, relevance, accuracy, depth, creativity, and level of detail of the response. Begin your evaluation by providing a short explanation.

\vspace{1em}
You will be given a high-quality reference answer and the assistant's answer. Be as objective as possible. You should first provide your explanation IN CHINESE, then you must rate the response on a scale of 1 to 10 by STRICTLY following the below MAPPING for the relation between the scores and response quality:
\vspace{1em}
\vspace{1em}

1) The score 1~2 stands for very chaotic or absence of answer, and the AI assissant completely failed to address the instructions. The gap between the AI assistant's answer and the high-quality reference answer is huge and insuperable.

\vspace{1em}
2) The score 3~4 indicates fragment-like responses from AI assistant's answer. It did not provide answers in proper grammar, fluency, or accuracy. There are obvious gaps between the high-quality reference answer and the AI assistant's response.

\vspace{1em}
3) The score 5~6 indicates for existence of minute disadvantage from the AI assistant's answer compared to the high-quality reference answer. Yet the AI assistant did provide an average answer. The AI assistant either did not fully satisfy instructions, or was somewhat short of helpfulness, relevance, depth, creativity, or detailedness. The disadvantages from the AI assistant's answer overwhelm its advantages.

\vspace{1em}
4) The score 7~8 indicates the AI assistant provided a good answer as well as the high-quality reference answer, satisfying the instruction, while addressing good helpfulness, relevance, accuracy, depth, creativity, and level of detail of the response. The AI assistant might have flaws compared to the reference answer, but that does not overwhelm the above advantages.

\vspace{1em}
5) The score 9~10 indicates the AI assistant responsed better than the provided reference answer in most aspects, fully achieved the instruction goal, and have unique advantages to the reference answer. Or the content of the reference answer can be completely covered.

\vspace{1em}
\vspace{1em}
You should give scores around 7 if you do not find obvious advantages or disadvantages. You should seriously consider the above guide before give lowest and highest scores such as 1 or 10, and avoid such situation if you do not have sound explanations.

\vspace{1em}
Avoid any positional biases and ensure that the order in which the responses were presented does not influence your decision. Do not allow the length of the responses to influence your evaluation. Do not favor certain names of the assistants. AND again, VERY IMPORTANTLY, after you provide your explanation IN CHINESE, you must rate the response strictly following this \textbf{FORMAT}: 

\vspace{1em}

Rating: [[score]]
\vspace{1em}

[Question]

\vspace{1em}

[Reference Answer]

\vspace{1em}

[Assistant's Answer]

    \vspace{1em}
    \end{mdframed}
    \caption{The prompt for DeepSeek-V3 as the rating model, which contains scoring criteria.}
     \label{DeepSeek-V3:prompt}
\end{figure*}

\end{document}